\documentclass[twoside,11pt]{article}

\usepackage{jmlr2e}

\usepackage[utf8]{inputenc} % allow utf-8 input
\usepackage[T1]{fontenc}    % use 8-bit T1 fonts
\usepackage{hyperref}       % hyperlinks
\usepackage{url}            % simple URL typesetting
\usepackage{booktabs}       % professional-quality tables
\usepackage{amsfonts}       % blackboard math symbols
\usepackage{nicefrac}       % compact symbols for 1/2, etc.
\usepackage{microtype}      % microtypography
\usepackage{xcolor}         % colors

\usepackage{todonotes}
\usepackage{amsfonts}
\usepackage{amsmath,amssymb}
\usepackage{nccmath}
\usepackage{mathrsfs}
\usepackage{enumitem}
\usepackage{subfigure}
\usepackage{adjustbox}
\usepackage{tabularx}
\usepackage{here}

\usepackage{lastpage}
\jmlrheading{27}{2026}{1-\pageref{LastPage}}{3/25; Revised
4/26}{5/26}{25-0509}{Yuka Hashimoto, Ayoub Hafid, Masahiro Ikeda, and Hachem Kadri}
\ShortHeadings{}{all authors LAST names}

\def\mcl#1{\mathcal{#1}}
\def\bracket#1{\left\langle #1\right\rangle}

\def\nn{\nonumber}
\def\opn{\operatorname}
\def\mr{\mathrm}

\def\alg{\mcl{A}}
\def\modu{\mcl{M}}

\def\bbracket#1{\big\langle #1\big\rangle}
\def\Bbracket#1{\bigg\langle #1\bigg\rangle}

\def\bc{\mathbf{c}}

\def\bG{\mathbf{G}}
\def\by{\mathbf{y}}

\def\eqref#1{(\ref{#1})}

\def\red#1{\textcolor{black}{#1}}
\def\blue#1{\textcolor{black}{#1}}
\def\green#1{\textcolor{black}{#1}}

\DeclareSymbolFont{EulerExtension}{U}{euex}{m}{n}
\DeclareMathSymbol{\euintop}{\mathop} {EulerExtension}{"52}
\DeclareMathSymbol{\euointop}{\mathop} {EulerExtension}{"48}

\newenvironment{mythm}[1][]{\medskip\par\noindent{\bfseries #1}\ \,\,\em}{\medskip\par}

\ShortHeadings{Spectral Truncation Kernels}{Hashimoto, Hafid, Ikeda, and Kadri}
\firstpageno{1}

\begin{document}

\title{Spectral Truncation Kernels: Noncommutativity in $C^*$-algebraic Kernel Machines}

\author{\name Yuka Hashimoto \email yuka.hashimoto@ntt.com \\
       \addr NTT, Inc.\\
       3-9-11, Midori-cho, Musashinoshi, Tokyo, 180-8585, Japan / \\
       Center for Advanced Intelligence Project, RIKEN\\
       1-4-1, Nihonbashi, Chuo-ku, Tokyo 103-0027, Japan
       \AND
       \name Ayoub Hafid \email hafid@ms.u-tokyo.ac.jp\\
       \addr Graduate School of Mathematical Sciences, The University of Tokyo\\
       3-8-1, Komaba, Meguro-ku, Tokyo 153-8914, Japan
       \AND 
       \name Masahiro Ikeda \email ikeda@ist.osaka-u.ac.jp\\
       \addr Graduate School of Information Science and Technology,
       The university of Osaka\\
       1-5, Yamadaoka, Suita-shi, Osaka 565-0871, Japan / \\
       Center for Advanced Intelligence Project, RIKEN\\
       1-4-1, Nihonbashi, Chuo-ku, Tokyo 103-0027, Japan
       \AND 
       \name Hachem Kadri \email hachem.kadri@lis-lab.fr\\
       \addr Department of Computer Science\\ 
       Aix-Marseille University, CNRS, LIS\\
       163, avenue de Luminy -- Case 901 F-13288 Marseille Cedex 9, France
       }

\editor{Lorenzo Rosasco}

\maketitle

\begin{abstract}%   <- trailing '%' for backward compatibility of .sty file
%$C^*$-algebra-valued kernels could pave the way for the next generation of kernel machines. To further our fundamental understanding of learning with $C^*$-algebraic kernels, we propose a new class of positive definite kernels based on the spectral truncation. We focus on kernels whose inputs and outputs are vectors or functions and generalize typical kernels by introducing the noncommutativity of the products appearing in the kernels. The noncommutativity induces interactions along the data function domain. We show that the proposed kernels fill the gap between existing separable and commutative kernels. We also propose a \red{combined model} to obtain a more flexible framework. The flexibility of the proposed class of kernels allows us to go beyond previous separable and commutative kernels, addressing two of the foremost issues regarding learning in vector-valued RKHSs, namely the choice of the kernel and the computational cost.
\red{A central question in vector- and function-valued learning is how to design kernels that capture both local and non-local interactions while remaining computationally tractable. Existing operator-valued kernels offer only partial answers: separable kernels are efficient but fail to model interactions across the function domain, while commutative kernels capture only pointwise structure. To address this, we propose spectral truncation kernels, a new class of positive definite kernels for vector- and function-valued learning based on spectral truncation and $C^*$-algebra.  By allowing noncommutative products in the kernel construction, the proposed kernels induce interactions across the data function domain and fill the gap between existing separable and commutative kernels. In addition, by using the $C^*$-algebraic framework, we reduce the computational cost compared to the existing vector-valued RKHS framework with operator-valued kernels.}
\end{abstract}

\begin{keywords}
  Kernel methods, $C^*$-algebra, Positive definite kernel, Ridge regression
\end{keywords}

\section{Introduction}\label{sec:intro}
Kernel methods have been 
{one of the most fundamental tools}  
in machine learning~\citep{scholkopf01,gretton07,hofmann08,kernelmean}.
They have been applied, for example, to
ridge regression, principal component analysis, and support vector machine.
Kernel methods are characterized by reproducing kernel Hilbert spaces (RKHSs), which are constructed by positive definite kernels.
Typical positive definite kernels include the polynomial kernel, Gaussian kernel, and Laplacian kernel.
Product kernels, which are constructed by the product of multiple kernels, have also been considered~\citep{scholkopf01,thomas08}.

Standard positive definite kernels are scalar-valued, and are well-suited to learn scalar-valued functions. 
Kernel methods for vector- and function-valued outputs have also been investigated~\citep{alvarez2012kernels,kadri16}. The kernels, in these cases, are instead operator-valued, and the associated feature space is vector-valued RKHSs (vvRKHSs)~\citep{kadri12,kadri16,minh16}.
There are at least two challenges for vvRKHS methods: the computational cost and choice of kernels.
A typical kernel is the separable kernel, which is defined by the product of a scalar-valued kernel and a positive semi-definite operator~\citep{alvarez2012kernels}.
Another typical kernel is the commutative kernel, which is defined only with the pointwise calculation of functions or vectors~\citep{hashimoto21}.
Although applying the separable and commutative kernels is computationally efficient, there is a crucial shortcoming for each kernel.
Separable kernels identify dependencies between input and output variables separately, and cannot reflect information of input variables properly to output variables.
The output is determined only by the global information of the input.
On the other hand, commutative kernels only identify the pointwise (completely local) dependencies.
Indeed, they are two extreme cases regarding the dependencies between input and output variables.
Several attempts have been made to construct kernels that go beyond separable and commutative kernels.
%\red{A typical nonseparable kernel is the transformable kernels, which is defined as $k(x,y)v(z)=\int_{\mcl{X}}\tilde{k}(S_z(x),S_t(y))v(t)\mr{d}t$}
A typical nonseparable kernel is the transformable \green{kernel}, which is characterized by a map that can incorporate the information of input variables with the output variables.
\citet{huusari21} proposed entangled kernels based on concepts from quantum computing, such as partial trace and entanglement.
\citet{hashimoto23_aistats} proposed to use the product of circulant matrices and general squared matrices to construct kernels. 
Using this kernel, one can generalize the convolution and capture the effect of interactions of different Fourier components.
However, if we need an $m$-dimensional vector-valued outputs with these kernels, then we have to construct an $mN$ by $mN$ Gram matrix, where $N$ is the sample size, and the computational cost is $O(m^3N^3)$ in general.
Thus, with vvRKHSs, to go beyond separable and commutative kernels, the computational cost is significant.

In this work, we address the two challenges of the computational cost and the choice of kernels by introducing a new class of kernels based on the framework of reproducing kernel Hilbert $C^*$-module~(RKHM)~\citep{hashimoto21}.
%
%\todo[inline]{Instead of the sentence above: In this work, we address these two challenges by introducing a new class of kernels based on the framework of reproducing kernel Hilbert $C^*$-module~(RKHM)~\citep{hashimoto21}}
%
RKHM is a generalization of RKHS by means of $C^*$-algebra.
$C^*$-algebra is a generalization of the space of complex values and has structures of the norm, product, and involution~\citep{murphy90,lance95}.
It unifies operators and functions.
In this framework, kernels are generalized to $C^*$-algebra-valued kernel functions and allow us to consider function-valued kernels, leading to function-valued Gram matrices.
%
%\todo[inline]{Instead of the sentence above: In this framework, reproducing kernels are generalized to $C^*$-algebra-valued kernel functions and allow us to consider function-valued kernels, leading to function-valued Gram matrices.}
%
By evaluating the values of the function-valued Gram matrix at $m$ different points, we obtain $m$ scalar-valued Gram matrices. 
This allows us to obtain an $m$-dimensional vector-valued outputs with the computational cost of $O(mN^3)$, which alleviates the dependency on $m$ from cubic
to linear compared to the case of vvRKHSs with nonseparable kernels such as transformable kernels. 
We summarize the difference between the proposed and existing kernels in Table~\ref{tab:summary}.

To obtain kernels going beyond the separable and commutative kernels with low computational cost, we propose a new class of $C^*$-algebra-valued positive definite kernels based on the spectral truncation, which has been discussed in the fields of noncommutative geometry and $C^*$-algebra~\citep{dandrea14,suijlekom21,connes21}.
%Our class of kernels is large and encompasses previous commutative kernel functions such as polynomial, product, and separable kernels.
%\todo[inline]{Instead of the sentence above: Our class of kernel is large and encompasses previous commutative kernel functions such as polynomial, product, and separable kernels.  }
%
The proposed kernels are parameterized by a natural number $n$ corresponding to the dimension of the truncated space.
They can be applied to both vector and functional inputs.
For vector inputs, we regard the elements of them as the values of functions.
We approximate the input functions on the $n$-dimensional truncated space and obtain $n$ by $n$ 
Toeplitz matrices, whose $(i,j)$-entry depends only on $i-j$.
{Thus, $n$ describes the resolution of the discretization, and we call it the truncation parameter.}
Indeed, $n$ describes the noncommutativity of the kernel.
We show that the proposed kernels converge to the commutative kernels as $n$ goes to infinity.
On the other hand, if $n=1$, then the proposed kernels are separable kernel.
Thus, we can control local and global dependencies through $n$.
%Second, the parameter $n$ controls the tradeoff between the representation power and the {model} complexity. 
%We show that if $n$ is small, then the representation power is low, and the model complexity is small.
%On the other hand, if $n$ is large, then the representation power is high, and the model complexity is large.
%The expectation loss for learning a function is controlled by the representation power and the complexity.
By introducing the parameter $n$ and setting $n$ to balance the information of gloal dependencies and that of the local dependencies, we obtain higher performance compared to the separable and commutative kernels.
We also propose a \blue{combined model} with the proposed kernels to make the framework more flexible and capture both local and global dependencies of the output on the input.

%{To analyze functional data, we have focused on how to {\em interpolate} observed data to obtain functions~\citep{ramsay05,wang16,hashimoto21}.
%On the other hand, the proposed kernels shed light on good effects of the {\em discretization}~(setting $n$ as a finite number) on learning process.}
In the setting of functional data~\citep{ramsay05,wang16,hashimoto21}, we have focused on how to {\em interpolate} observed data to obtain functions~\citep{ramsay05,wang16,hashimoto21}.
On the other hand, the proposed truncation kernels shed light on the good effects of {\em discretization}~(setting $n$ as a finite number) on the learning process.

Our contributions are summarized as follows:
%\begin{itemize}[leftmargin=*,nosep]
%\item  We propose spectral truncation kernels, a new class of function-valued positive definite kernels that go beyond separable and commutative kernels with low computational cost. (see Definition~\ref{def:kernels})
%\item The proposed kernels are based on the spectral truncation and are indexed by a truncation parameter, which adjusts the global and local interactions along the data function domain by introducing the noncommutativity into the learning process. (see Subsection~\ref{subsec:connection})
%\item We propose a \red{combined model} to make the framework more flexible (see Section~\ref{sec:deep}).
%\end{itemize} 

%\color{red}
\begin{itemize}[leftmargin=*,nosep]
\item We propose a new kernel family (see Definition~\ref{def:kernels}).
\item We study theoretical properties necessary for the applicability of a new kernel to practical learning tasks (see Sections~\ref{sec:kernel_trancation} and \ref{sec:generalization_bound}).
\item We conduct empirical validation on the tasks designed to showcase problems that could not be solved only with existing kernels (see Section~\ref{subsec:numexp}).
\end{itemize} 
%\color{black}

\green{The core contribution of this paper, the theory and properties of spectral truncation kernels, is documented in Sections \ref{sec:preliminaries}--\ref{sec:generalization_bound}. 
Section~\ref{sec:preliminaries} reviews the background on $C^*$-algebras, spectral truncation, and operator-valued kernels. 
Section~\ref{sec:kernel_trancation} introduces the proposed kernel class and discusses its positive definiteness. 
We also show how the truncation parameter $n$ bridges separable and commutative kernels. 
Section~\ref{sec:generalization_bound} provides a generalization bound of the models with proposed kernels.
Readers primarily interested in the kernel theory can focus on these sections. 
Section~\ref{sec:application} discusses the application of the proposed kernels to kernel ridge regression and analyzes computational cost. 
Section~\ref{sec:deep} proposes an extension of the framework by combining multiple spectral kernels.
This section is for readers interested in a more flexible, learnable variant. 
Section~\ref{sec:applications} provides potential applications of the proposed kernels.
Section~\ref{subsec:numexp} presents empirical evaluations of both the base and combined models.
Section~\ref{sec:kernel_generalization} shows that the proposed kernels can be generalized for more general settings.
Section~\ref{sec:conclusion} concludes the paper.}

\newcolumntype{C}{>{\centering\arraybackslash}X}
\begin{table}[t]%\vspace{-1.3cm}
\caption{Summary of the existing and the proposed kernels}\label{tab:summary}
\small
\begin{tabularx}{\textwidth}{c|C|C|C}
    \hline
     Type of kernels & Computational cost (ridge regression)&  Extraction of {\em local} information& Extraction of {\em global} information\\
     \hline
 Separable    & $O(mN^3)$ & $\times$ & $\checkmark$\\
 Commutative & $O(mN^3)$ & $\checkmark$ & $\times$\\
 Transformable & $O(m^3N^3)$ & $\checkmark$ & $\checkmark$\\
 {\bf Proposed} (with RKHM) & $O(mn^2N^2+mN^3)$ & $\checkmark$ & $\checkmark$\\
 \hline
\end{tabularx}
%\todo[inline]{In table 1, $\bigcirc$ means ok? If it is the case, we can use  $\checkmark$ instead.}
\end{table}

\section{Preliminaries}\label{sec:preliminaries}
We review some mathematical notions and existing studies.

\subsection{$C^*$-algebra and Reproducing kernel Hilbert $C^*$-module}\label{subsec:c_star}
$C^*$-algebra is a Banach space equipped with a product and an involution satisfying the $C^*$ identity (condition 3 below).
It is a natural generalization of the space of complex numbers.
\begin{definition}[$C^*$-algebra]~\label{def:c*_algebra}
A set $\alg$ is called a {\em $C^*$-algebra} if it satisfies the following conditions:\vspace{-.2cm}
\begin{enumerate}[itemsep=1pt,leftmargin=*]
 \item $\alg$ is an algebra over $\mathbb{C}$ and {equipped with} a bijection $(\cdot)^*:\alg\to\alg$ that satisfies the following conditions for $\alpha,\beta\in\mathbb{C}$ and $a,b\in\alg$:

 \leftskip=10pt
 $\bullet$ $(\alpha a+\beta b)^*=\overline{\alpha}a^*+\overline{\beta}b^*$,\qquad
 $\bullet$ $(ab)^*=b^*a^*$,\qquad
 $\bullet$ $(a^*)^*=a$.

 \leftskip=0pt
 \item $\alg$ is a normed space endowed with a norm $\Vert\cdot\Vert_{\alg}$, and for $a,b\in\alg$, $\Vert ab\Vert_{\alg}\le\Vert a\Vert_{\alg}\Vert b\Vert_{\alg}$ holds.
 In addition, $\alg$ is complete with respect to $\Vert\cdot\Vert_{\alg}$.

 \item For $a\in\alg$, the $C^*$ identity $\Vert a^*a\Vert_{\alg}=\Vert a\Vert_{\alg}^2$ holds.
\end{enumerate}
If there exists $a\in\alg$ such that $ab=b=ba$ for any $b\in\alg$, $a$ is called the unit and denoted by $1_{\alg}$.
\end{definition}
%In the following, we denote by $\alg$ a $C^*$-algebra.
In this paper, we focus on the \red{unital $C^*$ algebra ($C^*$-algebra with the unit)}, especially on $C^*$-algebra of continuous functions.
\begin{example}\label{ex:continuous_funcs}
Let $\mathbb{T}=\mathbb{R}/2\pi\mathbb{Z}$ be the torus and $C(\mathbb{T})$ be the space of continuous functions on $\mathbb{T}$.
Then, $\alg:=C(\mathbb{T})$ is a $C^*$-algebra by means of the product: $(cd)(z)=c(z)d(z)$, involution: $c^*(z)=\overline{c(z)}$, and norm: $\Vert c\Vert_{\alg}=\sup_{z\in\mathbb{T}}\vert c(z)\vert$ for $c,d\in\alg$.
The unit is the constant function $1_{\alg}\equiv 1$.
\end{example}

We now review basic notions regarding $C^*$-algebra.
In the following, let $\alg$ be a $C^*$-algbra.
\begin{definition}[Positive]~\label{def:positive}
An element $a$ of $\alg$ is called {\em positive} if there exists $b\in\alg$ such that $a=b^*b$ holds.
For $a,b\in\alg$, we {write} $a\le_{\alg} b$ if $b-a$ is positive.
%We denote by $\alg_+$ the subset of $\alg$ composed of all positive elements in $\alg$.
\end{definition}
\begin{definition}[Infimum and minimum]\label{def:sup}
For a subset $\mcl{S}$ of $\alg$, $a\in\alg$ is said to be a {\em lower bound} with respect to the order $\le_{\alg}$, if $a\le_{\alg} b$ for any $b\in\mcl{S}$.
Then, a lower bound $c\in\alg$ is said to be an {\em infimum} of $\mcl{S}$, if $a\le_{\alg} c$ for any lower bound $a$ of $\mcl{S}$. 
If $c\in\mcl{S}$, then $c$ is said to be a {\em minimum} of $\mcl{S}$.
\end{definition}
\red{Note that as the case of the standard order in $\mathbb{R}$, the minimum does not always exist.}
%
%absolute value, order, sup, unit

We now define RKHM.
Let $\mcl{X}$ be a non-empty set for data.
To construct an RKHM, we first introduce $\alg$-valued positive definite kernel.
\begin{definition}[$\alg$-valued positive definite kernel]\label{def:pdk_rkhm}
 An $\alg$-valued map $k:\mcl{X}\times \mcl{X}\to\alg$ is called a {\em positive definite kernel} if it satisfies the following conditions: \smallskip\\
%\begin{enumerate}[leftmargin=.35cm]
$\bullet$ $k(x,y)=k(y,x)^*$ \;for $x,y\in\mcl{X}$,\\
$\bullet$ $\sum_{i,j=1}^Nc_i^*k(x_i,x_j)c_j$ is \blue{positive} \;for $n\in\mathbb{N}$, $c_i\in\alg$, $x_i\in\mcl{X}$.
%\end{enumerate}
\end{definition}
Let $\phi: \mcl{X}\to\alg^{\mcl{X}}$ be the {\em feature map} associated with $k$, defined as $\phi(x)=k(\cdot,x)$ for $x\in\mcl{X}$ and let
$\modu_{k,0}=\{\sum_{i=1}^{N}\phi(x_i)c_i|\ N\in\mathbb{N},\ c_i\in\alg,\ x_i\in \mcl{X}\ (i=1,\ldots,N)\}$.
We can define an $\alg$-valued map $\bbracket{\cdot,\cdot}_{k}:\modu_{k,0}\times \modu_{k,0}\to\alg$ as
\begin{align*}
\Bbracket{\sum_{i=1}^{N}\phi(x_i)c_i,\sum_{j=1}^{M}\phi(y_j)d_j}_{k}=\sum_{i=1}^{N}\sum_{j=1}^{M}c_i^*k(x_i,y_j)d_j,
\end{align*}
which enjoys the reproducing property
$\bracket{\phi(x),f}_{k}=f(x)$
for $f\in\modu_{k,0}$ and $x\in \mcl{X}$.
\red{Note that this map satisfies ``$\bracket{f,f}_k=0$ $\Rightarrow$ $f=0$'' by the reproducing property. Thus, we can define the norm in $\modu_{k,0}$ as $\Vert f\Vert_k^2=\Vert\bracket{f,f}_k\Vert_{\alg}$.}
We can also define the $\alg$-valued absolute value $\vert \cdot\vert_k$ as $\vert f\vert_k^2=\bracket{f,f}_k$.
The {\em reproducing kernel Hilbert $\alg$-module~(RKHM)} $\modu_k$ associated with $k$ is defined as the completion of $\modu_{k,0}$ with respect to $\Vert\cdot\Vert_k$.  
See, for example, the references~\citep{lance95,murphy90,hashimoto21} for more details about $C^*$-algebra and RKHM.

\subsection{Spectral truncation on the torus}\label{subsec:spectral_truncation}
As we mentioned in Example~\ref{ex:continuous_funcs}, the product in $C(\mathbb{T})$ is commutative.
However, by approximating the multiplication of a function $x\in C(\mathbb{T})$ by a matrix, we can obtain a noncommutative product structure. 
Let $e_j$ be the Fourier function defined as $e_j(z)=\mr{e}^{\mr{i}jz}$ for $j\in\mathbb{Z}$ and $z\in\mathbb{T}$ and $M_x$ be the multiplication operator defined on $L^2(\mathbb{T})$ with respect to $x$.
Here, $\mr{i}$ is the imaginary unit, and $L^2(\mathbb{T})$ is the space of square-integrable functions on $\mathbb{T}$. 
Let $P_n$ be the orthogonal projection onto the $n$-dimensional subspace $\opn{Span}\{e_1,\ldots,e_n\}$.
We approximate $M_x$ by $P_nM_xP_n$, i.e., by truncating the spectrum.
Then, the $(j,l)$-entry of the representation matrix $R_n(x)\in\mathbb{C}^{n\times n}$ of $P_nM_xP_n$ is written as
\begin{align*}
R_n(x)_{j,l}=\bracket{e_j,M_xe_l}_{L^2(\mathbb{T})}
=\int_{\mathbb{T}}x(t)\mr{e}^{-\mr{i}(j-l)t}\mr{d}t.
\end{align*}
Since the $(j,l)$-entry of $R_n(x)$ depends only on $j-l$, $R_n(x)$ is a Toeplitz matrix~\citep{gray06}.
It is characterized only by a vector $r_n(x)\in\mathbb{C}^{2n-1}$, where the $(j-l)$th element of $r_n(x)$ is $R_n(x)_{j,l}$.
Note that $R_n(x)_{j,l}$ is regarded as the $(j-l)$th Fourier component of $x$, and the vector $r_n(x)$ is the coordinate of the function $x$ projected on the space $\opn{Span}\{e_{-(n-1)},\ldots,e_{n-1}\}$.

For a matrix $A\in\mathbb{C}^{n\times n}$, let $S_n(A)\in C(\mathbb{T})$ be the function defined as $S_n(A)(z)=(1/n)\sum_{j,l=0}^{n-1}A_{j,l}\mr{e}^{\mr{i}(j-l)z}$, where $A_{j,l}$ is the $(j,l)$-entry of $A$.
The map $S_n$ takes the representation matrix $R_n(x)$ back to a function that approximates the original function $x$.
Indeed, we have
\begin{align*}
S_n(R_n(x))(z)=\int_{\mathbb{T}}x(t)\frac{1}{n}\sum_{j,l=0}^{n-1}\mr{e}^{-\mr{i}(j-l)t}\mr{e}^{\mr{i}(j-l)z}\mr{d}t
=\int_{\mathbb{T}}x(t)\frac{1}{n}\sum_{j=0}^{n-1}\sum_{l=-j}^j\mr{e}^{\mr{i}l(z-t)}\mr{d}t
=x\ast F_n(z),
\end{align*}
where $F_n(t)=(1/n)\sum_{j=0}^{n-1}\sum_{l=-j}^j\mr{e}^{\mr{i}lt}$ is the Fej\'{e}r kernel and $\ast$ represents the convolution.
If $n=1$, then $F_n(z)=1$, and as $n$ goes to infinity, $F_n$ goes to the delta function.
More precisely, the following proposition holds~\citep{brandolini97,suijlekom21}, which implies that for each $z\in\mathbb{T}$, $S_n(R_n(x))(z)$ converges to $x(z)$ as $n\to\infty$.
%The pair of maps $R_n$ and $S_n$ is called the approximate order isomorphism.
\begin{proposition}
For each $z\in\mathbb{T}$, $x\ast F_n(z)\to x(z)$ as $n\to\infty$.
\end{proposition}

A generalization of the Fej\'{e}r kernel
on $\mathbb{T}$ to that on $\mathbb{T}^q$ with respect to the sum over a polyhedron has been theoretically investigated~\citep{Travaglini94,brandolini97}.
Let $P$ be a polyhedron and $jP=\{jz\,\mid\,z\in P\}$ for $j\in\mathbb{Z}$.
For $t\in\mathbb{T}^q$, the Fej\'{e}r kernel on $\mathbb{T}^q$ is defined as 
\begin{align}
F_n^{q,P}(t)=\frac1n\sum_{j=1}^{n-1}\sum_{r\in jP\bigcap \mathbb{Z}^q}\mr{e}^{\mr{i}r\cdot t}.\label{eq:general_fejer}
\end{align}

\subsection{Existing operator-valued kernels}\label{ap:operator-valued_kernels}
We review existing operator-valued kernels and discuss their advantages and shortcomings.
We especially focus on the case $\mcl{X}=\alg^d$ and the application of kernel ridge regression in this paper.
For more details about the operator-valued kernels, see, for example,~\citet{alvarez2012kernels}.

Typical existing kernels are summarized as follows:
\red{Let $\tilde{\mcl{X}}$ be a nonempty set,}
$\tilde{k}:\tilde{\mcl{X}}\times \tilde{\mcl{X}}\to\mathbb{C}$ be a complex-valued positive definite kernel, $\mu$ be a finite positive measure on $\mathbb{T}^2$, and $S:\mathbb{T}\times \mcl{X}\to\tilde{\mcl{X}}$.
\red{Let $\mathcal{B}(L^2(\mathbb{T}))$ be the space of bounded linear operators on $L^2(\mathbb{T})$.}
Consider a kernel $k:\mcl{X}\times \mcl{X}\to\mcl{B}(L^2(\mathbb{T}))$ defined by
\begin{align}
\bracket{u,k(x,y)v}_{L^2(\mathbb{T})}=\int_{\mathbb{T}^2}\red{\overline{u(z)}}\tilde{k}(S(z,x),S(t,y))v(t)\mr{d}\mu(t,z)\label{eq:op_valued_kernel}
\end{align}
for $u,v\in L^2(\mathbb{T})$.

In general, we have to construct a Gram matrix whose elements are operators.
Typically, if we need an output function evaluated at $m$ points, then we discretize the operator on $O(m)$-dimensional space and obtain an $O(mN)$ by $O(mN)$ Gram matrix.
Then, the computational cost for the kernel regression task for obtaining an output function evaluated at $m$ points is $O(N^3m^3)$.
Therefore, the computational cost is significant.
As a special case of the kernel~\eqref{eq:op_valued_kernel}, the following two kernels are efficient in the sense of the computational cost.

\paragraph{Separable kernel}
If we set $\tilde{\mcl{X}}=\mcl{X}$ and $S(t,x)=x$, then the kernel $k$ defined in Eq.~\eqref{eq:op_valued_kernel} is called a separable kernel.
In other words, let $A\in\mcl{B}(L^2(\mathbb{T}))$ be a Hermitian positive semi-definite operator defined as $\bracket{u,Av}_{L^2(\mathbb{T})}=\int_{\red{\mathbb{T}^2}}\overline{u(t)}v(z)\mr{d}\mu(z,t)$.
The kernel $k:\mcl{X}\times \mcl{X}\to\mcl{B}(L^2(\mathbb{T}))$ defined as
\begin{align*}
k(x,y)v(z)=\tilde{k}(x,y)Av(z)
\end{align*}
is called a separable kernel.

For separable kernels, the computation is reduced to the computation of the Gram matrix of $\tilde{k}$.
Thus, the computational cost for the kernel regression task for obtaining an output function evaluated at $m$ points is $O(N^3m)$, where $N$ is the sample size.
However, as we can see from the definition, it identifies dependencies between input and output variables separately and cannot reflect local information of input variables to output variables.
%In the framework of vvRKHSs, the output function with an input $x$ is in the form of $\sum_{i=1}^Nk(x,x_i)v_i$ for given samples $x_1,\ldots,x_N\in\mcl{X}$, and some $v_1,\ldots,v_N\in L^2(\mathbb{T})$.
Indeed, with separable kernels, we cannot specify the relationship between the value of the output function at $z$ and the value $x(z)$ of the input $x$ at $z$ although they have strong connection in many cases.
For example, if we have an input $[x_1,\ldots,x_d]\in\alg^d$ as $d$ time-series explanatory variables (i.e., $x_1,\ldots,x_d$ are functions with respect to time) and try to obtain an output function as a response variable, the values of $x_1(z),\ldots,x_d(z)$ at time $z$ is strongly related to the value of the output at time $z$.
In this case, the separable kernels are not suitable for extracting the relationship between $x_1(z),\ldots,x_d(z)$ and the values of the output at $z$.

\paragraph{Commutative kernel}
Let $\tilde{\mcl{X}}=\mathbb{C}$ and $S(t,x)=x(t)$.
\red{Let $\Delta :\mathbb{T}\to\mathbb{T}^2$ be defined as $\Delta (z)=(z,z)$ and let $\mu=\Delta_{\sharp}m$, where $m$ is the normalized Haar measure on $\mathbb{T}$ and $_{\sharp}$ represents the push forward.
Then, we have $\bracket{u,k(x,y)v}_{L^2(\mathbb{T})}=\int_{\mathbb{T}}\red{\overline{u(z)}}\tilde{k}(x(z),y(z))v(z)\mr{d}z$.}
Thus, we have  
\begin{align*}
k(x,y)v(z)=\tilde{k}(x(z),y(z))v(z).
\end{align*}
We call this kernel the commutative kernel.

For commutative kernels, the computation is reduced to the computation at each $z$.
Thus, the computational cost for the kernel regression task for obtaining an output function evaluated at $m$ points is $O(N^3m)$.
However, as we can see from the definition, it
only identifies completely local relationship between the input function and the output function.
The value of the output function at $z$ is determined only with the value of the input function at $z$.
In the same example as separable kernels, the values of $x_1(z),\ldots,x_d(z)$ at time $z$ is strongly related to the value of the output at time $z$, but may also be related to $y(z+t)$ for $t\in [-T,T]$ for a small number $T$.
In this case, the commutative kernels are not suitable for extracting the relationship between $x_1(z),\ldots,x_d(z)$ and the values of the output around $z$, not only exactly at $z$.

As we have seen, although the separable and commutative kernels are computationally efficient, they are two extreme cases regarding the description of the relationship between the input and output functions; the separable kernels identify dependencies between input and output variables separately (only with global information) and the commutative kernels only identify the commutative (completely local) dependencies.
The following existing kernels fill a gap between the above two kernels, but computationally expensive.
\paragraph{Transformable kernel}
If we set $\mu$ as the Lebesgue measure, then the kernel $k$ defined in Eq.~\eqref{eq:op_valued_kernel} is called the transformable kernel~\red{\citep{caponnetto08a}}.
In this case, the kernels can identify the dependency between input and output variables through the map $S$.
However, the computational cost for the kernel regression task for obtaining an output function evaluated at $m$ points is $O(N^3m^3)$ as we discussed at the beginning of this section.
In addition, we have to determine the map $S$, but it is not easy to interpret in general.

\paragraph{Combination of separable and transformable kernels}
Considering the sum of separable and transformable kernels and the product of these kernels have also been proposed.
However, we have the same shortcomings of separable and transformable kernels.

\section{$C^*$-algebra-valued Positive Definite Kernel with Spectral Truncation}\label{sec:kernel_trancation}
In the following, we set the $C^*$-algebra $\alg$ as $C(\mathbb{T})$.
To obtain vector- or functional-valued outputs, applying vvRKHSs has been investigated.
However, as we stated in Section~\ref{sec:intro}, 
\green{in problems where the output function exhibits nontrivial dependencies along the function domain that separable kernels cannot capture by construction, going beyond the separable kernel can yield improved performance. 
However, existing nonseparable kernels such as transformable kernels incur a computational cost of $O(m^3N^3)$ for $m$-dimensional outputs with $N$ samples, making them impractical (See Subsection~\ref{ap:operator-valued_kernels}).}
%to go beyond the separable kernel and reach higher performance, the computational cost scales as $O(m^3N^3)$ for $m$-dimensional outputs with $N$ samples.
%
%\todo[inline]{sentence above: "reach" instead of "obtain". "scales as $O(m^3N^3)$" instead of "depends on $m^3N^3$" }
%
To go beyond the separable and commutative kernels with lower computational cost, we consider function-valued ($\alg$-valued) kernels and RKHMs.
As we will discuss in the last part of Section~\ref{sec:application}, the computational cost depends on $mN^3$ if we use an $\alg$-valued kernel.

%\citet{hashimoto21} proposed to use the $\alg$-valued kernel $k$ where $k(\cdot,\cdot)(z)$ is a complex-valued positive definite kernel for all $z\in\mathbb{T}$.
We first combine existing kernels~\citep{hashimoto21,hashimoto23_aistats,scholkopf01,kadri12,alvarez2012kernels} and obtain the following typical commutative $\alg$-valued kernels.
Here, we consider the case where the set $\mcl{X}$ for data is $\mcl{X}\subseteq \alg^d$.
%The following three kinds of $\alg$-valued kernels are often considered in existing studies~\cite{hashimoto23_aistats,hashimoto23-deeprkhm,scholkopf01,kadri12,alvarez2012kernels}.
\begin{example}\label{ex:standard_kernels}
For $x=[x_1,\ldots,x_d],y=[y_1,\ldots,y_d]\in \alg^d$ and $z\in\mathbb{T}$,
\begin{enumerate}
    \item polynomial kernel $k^{\opn{poly},q}(x,y)(z)=\sum_{i=1}^d\alpha_i(\overline{x_i(z)}y_i(z))^q$,
    \item product kernel $k^{\opn{prod},q}(x,y)(z)=\prod_{j=1}^q\overline{\tilde{k}_{1,j}(x(z),y(z))}\tilde{k}_{2,j}(x(z),y(z))$.
    %\item separable kernel $k^{\opn{sep},q}(x,y)=\tilde{k}(x,y)\prod_{j=1}^q(a_j^*)^qa_j^q$.
\end{enumerate}
Here, $q\in\mathbb{N}$ is the degree of the products in the kernel, $\alpha_i\ge 0$ is the parameter of $k^{\opn{poly},q}$, and for $i=1,2$ and $j=1,\ldots,q$, $\tilde{k}_{i,j}:\mathbb{C}\times\mathbb{C}\to\mathbb{C}$ 
%and $\tilde{k}:\alg^d\times\alg^d\to\mathbb{C}$ 
is a complex-valued continuous positive definite kernel.
%In addition, $a_j\in\alg$ is the parameter of $k^{\opn{sep},q}$.
\end{example}

We propose new $\alg$-valued positive definite kernels by generalizing the typical kernels in Example~\ref{ex:standard_kernels}.
The kernels in Example~\ref{ex:standard_kernels} involve the product of functions in $\alg$.
For $x,y\in \alg$, the product of $x$ and $y$ is commutative and is just the pointwise product defined as $(xy)(z)=x(z)y(z)$ (See Example~\ref{ex:continuous_funcs}).
Thus, as we also discussed in Subsection~\ref{ap:operator-valued_kernels}, the commutative kernels defined in Example~\ref{ex:standard_kernels} do not induce the interactions along $z\in\mathbb{T}$, that is, $k(x,y)(z)$ is determined only by the values $x(z)$ and $y(z)$ at $z$.
%We call the kernel whose output at $z$ is determined only with $x(z)$ and $y(z)$ a commutative kernel.
If the values $x(z_1)$ and $x(z_2)$, or $x(z_1)$ and $y(z_2)$ for different $z_1$ and $z_2$ are related, then the construction of the kernels only with this commutative product cannot extract that relationship.
On the other hand, if we introduce the truncation parameter $n$ and transform $x$ and $y$ into Toeplitz matrices using the map $R_n$ defined in Subsection~\ref{subsec:spectral_truncation}, then the product of $R_n(x)$ and $R_n(y)$ is noncommutative.
Focusing on this feature, we define $\alg$-valued kernels based on Example~\ref{ex:standard_kernels}, but they have interactions along $z\in\mathbb{T}$.

\begin{definition}\label{def:kernels}
With the notations in Example~\ref{ex:standard_kernels}, for $x=[x_1,\ldots,x_d],y=[y_1,\ldots,y_d]\in \alg^d$, let
\begin{align*}
k_n^{\opn{poly},q}(x,y)&=S_n\bigg(\sum_{i=1}^d\alpha_i(R_n(x_i)^*)^qR_n(y_i)^q\bigg),\\
k_n^{\opn{prod},q}(x,y)&=S_n\bigg(\prod_{j=1}^qR_n(\tilde{k}_{1,j}(x,y))^*\prod_{j=1}^qR_n(\tilde{k}_{2,j}(x,y))\bigg).
%k_n^{\opn{sep},q}(x,y)&=\tilde{k}(S_n(R_n(x)),S_n(R_n(y)))S_n\bigg(\prod_{j=1}^qR_n(a_j)^*\prod_{j=1}^qR_n(a_j)\bigg).
%k(x,y)=\sum_{i=1}^dc_i(x_i^*)^qy_i^q
\end{align*}
Here, $A^*$ for a matrix $A$ is the adjoint, and we denote by $\tilde{k}_{i,j}(x,y)$ the map $z\mapsto \tilde{k}_{i,j}(x(z),y(z))$.
\end{definition}

\begin{remark}
Although the inputs of the kernels in Example~\ref{ex:standard_kernels} and Definition~\ref{def:kernels} are functions, we can also deal with vector inputs.
For vector inputs, we can regard them as values of functions, and approximate the integral by the discrete sum.
\end{remark}

In Definition~\ref{def:kernels}, to construct the kernels, we first project functions in $\alg$ onto the space $\opn{Span}\{e_{-(n-1)},\ldots,e_{n-1}\}$ using the map $R_n$ and obtain Toeplitz matrices.
Then, we consider the product of the Toeplitz matrices.
Note that the product of two Toeplitz matrices is not always a Toeplitz matrix.
Finally, we apply $S_n$ to the matrix and take it back to the space $\opn{Span}\{e_{-(n-1)},\ldots,e_{n-1}\}$ and obtain the output function of the kernel.
%(Note that we can regard $x\in\alg$ projected on $\opn{Span}\{e_{-n},\ldots,e_n\}$ as the multiplication operator $M_x$ projected on $\opn{Span}\{e_1,\ldots,e_n\}$.)
Figure~\ref{fig:kernel_overview} schematically shows the construction of the simplest kernel $k_n^{\opn{poly},1}(x,y)=S_n(R_n(x)^*R_n(y))$ defined in Definition~\ref{def:kernels}.

\begin{figure}
    \centering
    \includegraphics[scale=0.4]{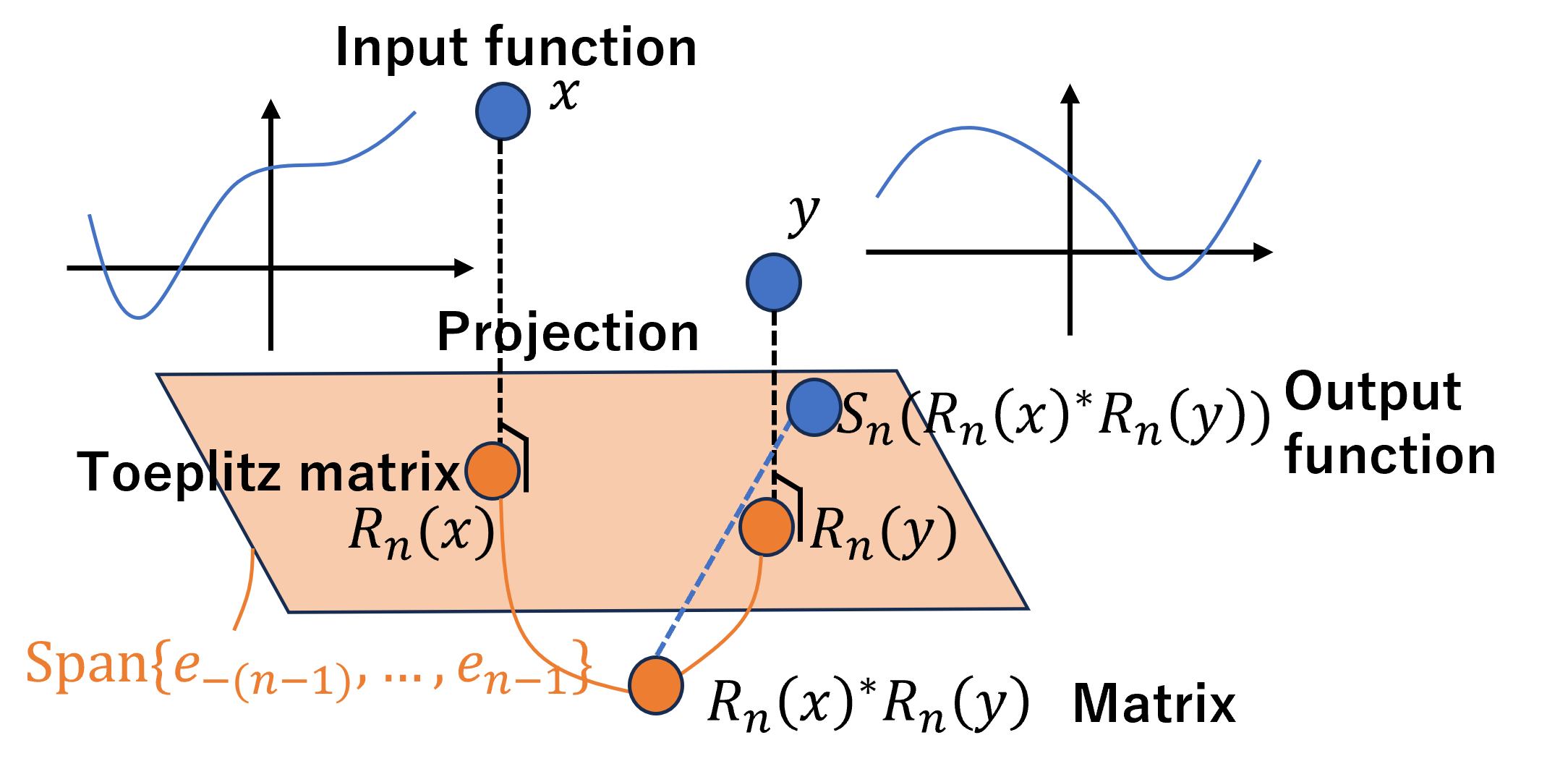}\vspace{-.4cm}
    \caption{Overview of the construction of the simplest kernel $k_n^{\opn{poly},1}(x,y)=S_n(R_n(x)^*R_n(y))$}
    \label{fig:kernel_overview}
\end{figure}

By the definitions of $S_n$ and $R_n$, we can show the following identity for $x_1,\ldots,x_q,y_1,\ldots,y_q\in\alg$~(see the proof of Theorem~\ref{thm:convergence} in Appendix~\ref{ap:proof}): 
\begin{align}
S_n\bigg(\prod_{j=1}^qR_n(x_j)^*\prod_{j=1}^qR_n(\red{y_j})\bigg)(z)=\int_{\mathbb{T}^{2q}}\!\!\!\!\!\!\overline{x_1(t_1)}\cdots\overline{x_q(t_q)}y_1(t_{q+1})\cdots y_q(t_{2q})F_n^{2q,P}(z\mathbf{1}-t)\mr{d}t,\label{eq:kernel_fejer}
\end{align}
where $t=[t_1,\ldots,t_{2q}]$, $\mathbf{1}=[1,\ldots,1]$, $P=\{[r_1,\ldots,r_{2q}]\in\mathbb{R}^{2q}\,\mid\,\vert \sum_{i=l}^mr_i\vert\le 1,\ l\le m\}$, and $F_n^{q,P}$
is the Fej\'{e}r kernel on $\mathbb{T}^q$ defined in Eq.~\eqref{eq:general_fejer}. 
Eq.~\eqref{eq:kernel_fejer} implies that the value of $S_n\big(\prod_{j=1}^qR_n(x_j)^*\prod_{j=1}^qR_n(y_j)\big)$ at $z$ is determined not only by $x_1(z),\ldots,x_q(z),y_1(z),\ldots,y_q(z)$, but also by $x_1(t_1),\ldots,x_q(t_q),y_1(t_{q+1}),\ldots,y_q(t_{2q})$ with different values $t_1,\ldots,t_{2q}$ from $z$.
Note that $F_n^{q,P}(t)$ is a real-valued function since for $r\in P$, we have $-r\in P$.
Thus, if a kernel in Example~\ref{ex:standard_kernels} is real-valued, then the corresponding kernel in Definition~\ref{def:kernels} is also real-valued.

\subsection{Connection between the proposed kernel and existing kernels}\label{subsec:connection}
The proposed function-valued kernels combined with the framework of RKHMs fill the gap between separable and commutative kernels with a computational cost lower than that of transformable kernels.
The proposed kernels are separable if $n=1$ and commutative if $n=\infty$.
Indeed, we have the following observation:

\paragraph{The case of $n=1$} The proposed kernels $k_n^{\opn{poly},q}$ and $k_n^{\opn{prod},q}$ are equivalent to separable kernels:
\begin{align*}
k_1^{\opn{poly},q}(x,y)(z)&=\sum_{i=1}^d\alpha_i\bigg(\int_{\mathbb{T}}\overline{x_i(t)}\mr{d}t\bigg)^q\bigg(\int_{\mathbb{T}}y_i(t)\mr{d}t\bigg)^q,\\
k_1^{\opn{prod},q}(x,y)(z)&=\prod_{j=1}^q\int_{\mathbb{T}}\overline{\tilde{k}_{1,j}(x(t),y(t))}\mr{d}t\prod_{j=1}^q\int_{\mathbb{T}}\tilde{k}_{2,j}(x(t),y(t))\mr{d}t.
%k_1^{\opn{sep},q}(x,y)(z)&=\tilde{k}\bigg(\int_{\mathbb{T}}x(t)\mr{d}t,\int_{\mathbb{T}}y(t)\mr{d}t\bigg)\prod_{j=1}^q\int_{\mathbb{T}}\overline{a_j(t)}\mr{d}t\prod_{j=1}^q\int_{\mathbb{T}}a_j(t)\mr{d}t.
\end{align*}
We can see that the output of the kernel does not depend on the variable $z$.
For example, the output function with an input $x$ of the kernel ridge regression is in the form of $\sum_{i=1}^Nk(x,x_i)(z)c_i(z)$ for given samples $x_1,\ldots,x_N\in\mcl{X}$ and some $c_1,\ldots,c_N\in C(\mathbb{T})$ (See \citet{hashimoto23_aistats} or Section~\ref{sec:application}).
Thus, the kernels identify dependencies between input and output functions separately.
We can only capture global information of the input $x$ for obtaining the output.

\paragraph{The case of $n=\infty$}
In the next section (Theorem~\ref{thm:convergence}), we will show that the proposed kernels $k_n^{\opn{poly},q}$ and $k_n^{\opn{prod},q}$ converge to commutative kernels $k^{\opn{poly},q}$ and $k^{\opn{prod},q}$ as $n$ goes to infinity.
Thus, the proposed kernels are equivalent to commutative kernels:
\begin{align*}
k_{\infty}^{\opn{poly},q}(x,y)(z)&=\sum_{i=1}^d\alpha_i(\overline{x_i(z)}y_i(z))^q,\\
k_{\infty}^{\opn{prod},q}(x,y)(z)&=\prod_{j=1}^q\overline{\tilde{k}_{1,j}(x(z),y(z))}\tilde{k}_{2,j}(x(z),y(z)).
\end{align*}
The value of the output of the kernels at $z$ are determined only by $x(z)$ and $y(z)$. The kernels identify completely local dependencies between input and output functions.

\paragraph{The case of $1<n<\infty$}
If $n$ is small, then we focus more on global information, and if $n$ is large, we focus more on local information.
Indeed, as we discussed in Eq.~\eqref{eq:kernel_fejer}, we have
\begin{align}
\begin{aligned}
k_{n}^{\opn{poly},q}(x,y)(z)&=\int_{\mathbb{T}^{2q}}\overline{x_1(t_1)}\cdots\overline{x_q(t_q)}y_1(t_{q+1})\cdots y_q(t_{2q})F_n^{2q,P}(z\mathbf{1}-t)\mr{d}t,\\
k_{n}^{\opn{prod},q}(x,y)(z)&=\int_{\mathbb{T}^{2q}}\prod_{j=1}^q\overline{\tilde{k}_{1,j}(x(t_j),y(t_j)}\tilde{k}_{2,j}(x(t_{q+j}),y(t_{q+j}))F_n^{2q,P}(z\mathbf{1}-t)\mr{d}t,
\end{aligned}\label{eq:kernel_fj}
\end{align}
where we assume $d=1$ for simplify the notation.
The proposed kernels are described by the convolution of the input function and the F\'{e}jer kernel $F_n^{2q,P}$.
If $n=1$, then $F_n^{2q,P}(z)=1$, and as $n$ goes to infinity, $F_n^{2q,P}$ goes to the delta function.
This means that if $n$ is small, then the convolution with $F_n^{2q,P}$ extract global information more than local information.
On the other hand, if $n$ is large, then the convolution with $F_n^{2q,P}$ extract local information more than global information.
We illustrate how $F_n^{2,P}$ changes along $n$ in Figure~\ref{fig:fejer}.

\begin{figure}[t]
    \centering
    \includegraphics[scale=0.42]{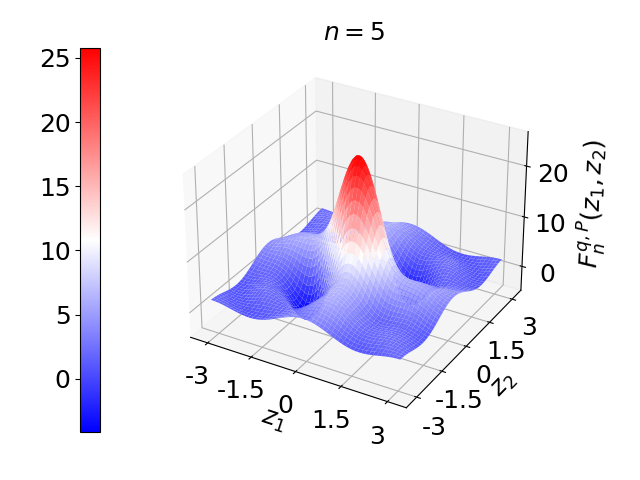}
    \includegraphics[scale=0.42]{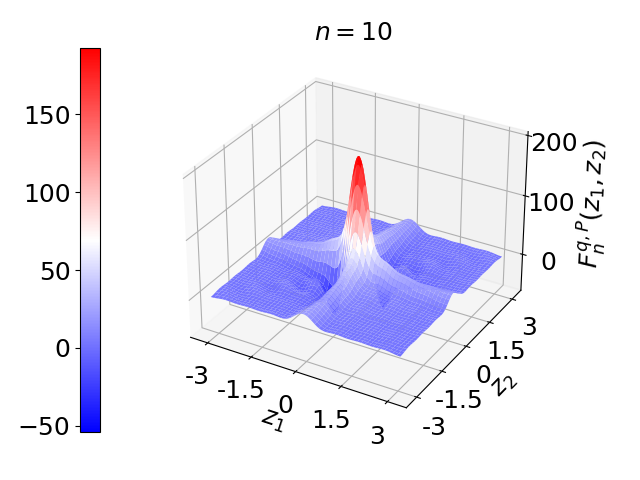}
    \includegraphics[scale=0.42]{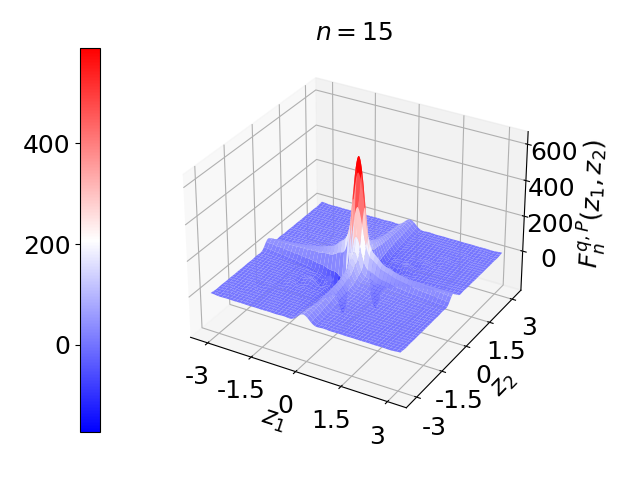}
    \caption{F\'{e}jer kernel $F_n^{2,P}$ for $n=5,10,15$}
    \label{fig:fejer}
\end{figure}

\begin{remark}
We can determine optimal $n$ in the sense of the dependencies by observing the F\'{e}jer function $F_n^{2q,P}$.
Since the proposed kernel is defined by the convolution of the input function with the F\'{e}jer kernel, the volume of the region where the value of the F\'{e}jer kernel is sufficiently large corresponds to the range of local dependencies.
Thus, if we have a information of the local dependencies, then we can choose $n$ based on the values of $F_n^{2q,P}$.
We also discuss a generalization of the proposed kernel to kernels with a continuous value parameter, instead of the discrete value parameter $n$, in Subsection~\ref{subsec:generalization_heat}. 
This generalization enables us to use an optimization method to learn a best possible parameter.
\end{remark}

\begin{remark}
The proposed kernels are composed of complex-valued kernels $\tilde{k}_{i,j}$ in Definition~\ref{def:kernels}. 
We can choose any kernel for $\tilde{k}_{i,j}$, and the properties of the proposed kernels depend also on that choice. 
If we choose a kernel with parameters as $\tilde{k}_{i,j}$ (such as a weighted sum of multiple kernels), then by optimizing the parameters, we can obtain a better kernel for given data or tasks. 
In addition, for function-valued kernels, in the same manner as the complex-valued kernels, the weighted sum of positive definite kernels and product of positive definite kernels are also a positive definite kernel. 
Thus, we can consider multiple proposed kernels and combine them with weight parameters. 
In that case, we can also optimize the parameters to obtain a better kernel for given data or tasks.
\end{remark}

\paragraph{Computational cost}
As we will discuss in the last paragraph of Section~\ref{sec:application}, the computational cost for the kernel regression task for obtaining an output function evaluated at $m$ points with the proposed kernel $k_n^{\opn{poly},q}$ or $k_n^{\opn{prod},q}$ is $O((q+m)n^2N^2+mN^3)$.
Thus, if $(q+m)n^2<m^3N$, then the proposed kernels are more computationally efficient than nonseparable and noncommutative kernels, such as the transformable kernels.

\subsection{Convergence and interactions}\label{subsec:convergence}
We show theoretical results that support the discussion in Subsection~\ref{subsec:connection}.
We first show that $k_n^{\opn{poly},q}$ and $k_n^{\opn{prod},q}$ defined in Definition~\ref{def:kernels} converge to $k^{\opn{poly},q}$ and $k^{\opn{prod},q}$ defined in Example~\ref{ex:standard_kernels} as $n$ goes to infinity, respectively.
%In this sense, the kernels $k_n^{\opn{poly},q}$, $k_n^{\opn{prod},q}$, and $k_n^{\opn{sep},q}$ are generalizations of typical kernels.
%In addition, note that if $n=1$, then $R_n$ is just the averaging operation, and the output of the kernels are constant with respect to $z$.
%Thus, the proposed kernels are separable if $n=1$ and are commutative if $n=\infty$.
%See Appendix~\ref{ap:proposed_kernel} for more details.
%In this sense, the proposed kernels fill a gap between the separable and commutative kernels.
%these convergences imply that the interactions along $z\in\mathbb{T}$ in the kernels become smaller as $n$ grows and they disappear at $n=\infty$.
In the following, all the proofs of theoretical statements are documented in Appendix~\ref{ap:proof}.

\begin{theorem}\label{thm:convergence}
For $x,y\in\alg^d$ and $z\in\mathbb{T}$, $k_n^{\opn{poly},q}(x,y)(z)\to k^{\opn{poly},q}(x,y)(z)$ and $k_n^{\opn{prod},q}(x,y)(z)\to k^{\opn{prod},q}(x,y)(z)$ as $n\to\infty$.
%If $x$ and $y$ are differentiable, then $k_n^{\opn{sep},q}(x,y)(z)\to k^{\opn{sep},q}(x,y)(z)$.
\end{theorem}

As we can see in Eq.~\eqref{eq:kernel_fejer}, if the kernel involves $q$ products of Toeplitz matrices generated by $R_n$, it is represented by a Fej\'{e}r kernel on $\mathbb{T}^q$.
To show Theorem~\ref{thm:convergence}, we apply the following lemmas.
\begin{lemma}\label{lem:index}
For $m\in\mathbb{N}$ and $j=0,\ldots 2q$, let $Q_j^m=\{r'=[r_1',\ldots,r_{2q}']\in\mathbb{R}^{2q}\,\mid\, r_i'=r_i-r_{i-1}\ (i=1,\ldots,2q),\ 0\le r_i\le m\ (i=0,\ldots,2q),\ r_j=m\}$ and $P=\{r=[r_1,\ldots,r_{2q}]\in\mathbb{R}^{2q}\,\mid\,\vert \sum_{i=l}^kr_i\vert\le 1,\ l\le k\}$.
Then, we have $mP=\bigcup_{j=0}^{2q}Q_j^m$.
\end{lemma}
\begin{lemma}\label{lem:fejer_convergence_general}
Let $P$ be a convex polyhedron.
Let $F_n^{q,P}$ be the Fej\'{e}r kernel on $\mathbb{T}^{q}$ defined as Eq.~\eqref{eq:general_fejer}.
Then, for any $g\in C(\mathbb{T}^q)$ and any $z\in\mathbb{T}^q$, $g\ast F_n^{q,P}(z)\to g(z)$ as $n\to \infty$.
\end{lemma}
Since the product of Toeplitz matrices in $k_n^{\opn{poly},q}$ and $k_n^{\opn{prod},q}$ are represented by the sum over the indices in $\bigcup_{j=0}^{2q}Q_j^m$,  Lemma~\ref{lem:index} explains why the Fej\'{e}r kernel with the polyhedron $P=\{r\in\mathbb{R}^{2q}\,\mid\,\vert \sum_{i=l}^kr_i\vert\le 1,\ l\le k\}$ appears in Eq.~\eqref{eq:kernel_fejer}.
Lemma~\ref{lem:fejer_convergence_general} generalizes the convergence of the Fej\'{e}r kernel on $\mathbb{T}$ to that on $\mathbb{T}^q$, which is proved by~\citet{brandolini97}.
We set $g(z)=\overline{x_1(z_1)}\cdots\overline{x_q(z_q)}y_1(z_{q+1})\cdots y_q(z_{2q})$ for $z=[z_1,\ldots,z_{2q}]$, based on Eq.~\eqref{eq:kernel_fejer}, and apply Lemmas~\ref{lem:index} and \ref{lem:fejer_convergence_general} to show Theorem~\ref{thm:convergence}.
See Appendix~\ref{ap:proof} for more details.
%By the convergence in Theorem~\ref{thm:convergence}, we can see that the interactions along $z\in\mathbb{T}$ in the kernels become small as $n$ glows since they disappear at $n=\infty$.
Theorem~\ref{thm:convergence} implies that the interactions along $z\in\mathbb{T}$ in the kernels become small as $n$ grows.
This is because as $n$ goes to infinity, $F_n^{2q,P}$ goes to the delta function and by taking the convolution with the input function, it focuses more on local relationships between input and output functions as $n$ grows.
Regarding the relationship between $q$ and the interactions see Appendix~\ref{ap:q_interactions}.

\subsection{Positive definiteness}\label{subsec:positive_definitness}
To construct an RKHM, which is introduced in Subsection~\ref{subsec:c_star}, the kernel $k$ has to be positive definite.
Thus, we investigate the positive definiteness of the proposed kernels.
Regarding $k_n^{\opn{poly},q}(x,y)$, we can show the positive definiteness since $x_i$ and $y_i$ are separated by the product.
%Regarding $k_n^{\opn{sep}}(x,y)$, since $x$ and $y$ depend only on $\tilde{k}$, we can use the positive definiteness of $\tilde{k}$.
\begin{proposition}\label{prop:pd_poly_sep}
The kernel $k_n^{\opn{poly},q}$ is a positive definite kernel.
\end{proposition}

As for $k_n^{\opn{prod},q}(x,y)$, it depends on $x$ and $y$ through $\tilde{k}_{i,j}$, and we cannot separate $x$ and $y$ as products.
Thus, we cannot show its positive definiteness.
However, we can modify the kernel to become positive definite as follows.
\begin{proposition}\label{prop:pd_prod}
Let $\beta_n\ge-\min_{z\in\red{\mathbb{T}^{2q}}}F_n^{2q,P}(z)$.
Then, $\hat{k}_n^{\opn{prod},q}$ defined below is a positive definite kernel.
\begin{align*}
\hat{k}_n^{\opn{prod},q}(x,y)=k_n^{\opn{prod},q}(x,y)+\beta_n\int_{\mathbb{T}^{2q}}\prod_{j=1}^q\overline{\tilde{k}_{1,j}(x(t_j),y(t_j))}\tilde{k}_{2,j}(x(t_{q+j}),y(t_{q+j}))\mr{d}t.
\end{align*}
\end{proposition}
\red{Note that by the compactness of $\mathbb{T}^{2q}$ and the continuity of $F_n^{2q,P}$, the minimum always exists.}
To set the value of the parameter $\beta_n$, we have the following bound.
\begin{lemma}\label{lem:fejer_bound}
The F\'{e}jer kernel $F_n^{q,P}$ is bounded as $\vert F_n^{q,P}(z)\vert\le n^q$.
\end{lemma}
\begin{remark}\label{rmk:positive_definiteness}
We can set $\beta_n$ in $\hat{k}_n^{\opn{prod},q}$ as $n^q$ to guarantee the positive definiteness.
However, even if we set a smaller $\beta_n$, $\hat{k}_n^{\opn{prod},q}$ may become positive definite.
In addition, considering non-positive kernels has also been investigated~\citep{ong04,canu05}.
Indeed, in practical computations in Section~\ref{subsec:numexp}, the algorithms can work even if we set $\beta_n<n^{2q}$.
Deriving a tighter bound of $\beta_n$ or developing the theory for non-positive kernels for RKHMs is future work.
\end{remark}

%\color{blue}
\begin{remark}
Although the kernel $\hat{k}^{\operatorname{prod},q}_n$ itself does not converge to the product kernel in Example~\ref{ex:standard_kernels}, its first term converges to the product kernel.
In addition, the second term can be written as
\begin{align*}
\beta_n\int_{\mathbb{T}^{2q}}\prod_{j=1}^q\overline{\tilde{k}_{1,j}(x(t_j),y(t_j))}\tilde{k}_{2,j}(x(t_{q+j}),y(t_{q+j}))\mr{d}t
=\beta_nk_1^{\operatorname{prod}}(x,y),
\end{align*}
which means $\hat{k}^{\operatorname{prod},q}_n$ is still a spectral truncation kernel.
\end{remark}
%\color{black}

\section{Generalization Bound}\label{sec:generalization_bound}

We now focus on what impact the truncation parameter $n$ has on generalization.
%We provide a generalization bound {allowing us to observe} the tradeoff between the representation power and the {model} complexity associated with the proposed kernels.
According to Lemma 4.2 in~\citet{mohri18}, the generalization bound is described by the Rademacher complexity.
We apply this lemma and derive a generalization bound.

Let \red{$(\Omega,\mcl{B},\mu)$} be a probability space.
Let $X_1,\ldots,X_N$ and $Y_1,\ldots,Y_N$ be samples from a distributions of $\alg_0^d$-valued random variable $X$ and $\alg_1$-valued random variable $Y$ on $\Omega$, respectively (i.e., for $z\in\mathbb{T}$, $X_i(z)$ is a sample from the distribution of $X(z)$).
Here, \red{$\alg_0$ and $\alg_1$ are subsets of $C(\mathbb{T},\mathbb{R})$, and $C(\mathbb{T},\mathbb{R})$ is the space of real-valued continuous functions on $\mathbb{T}$}.
Let $\mr{E}$ be the Bochner integral on $\Omega$ with respect to $\mu$.
Let $B>0$ and $\mcl{F}=\{f\in\blue{\modu_{k_n}}\,\mid\,\blue{\Vert f\Vert_{k_n}}\le B,\ \red{f\mbox{ is Bochner-measurable and Bochner-integrable}}\}$ \blue{for a kernel $k_n$}.
Let $g:\mathbb{R}\times\mathbb{R}\to\mathbb{R}_{+}$ be a \red{bounded} error function.
Assume there exists $L>0$ such that for $y\in\alg_1$ and $z\in\mathbb{T}$, $x\mapsto g(x,y(z))$ is $L$-Lipschitz continuous, \red{that is, $\vert g(x_1,y(z))-g(x_2,y(z))\vert\le L\vert x_1-x_2\vert$ for $x_1,x_2\in \mathbb{R}$}.
We derive the following generalization bound for the kernels defined in Section~\ref{sec:kernel_trancation} based on \citet{mohri18,hashimoto23_aistats}
Here, to adapt to the generalization bound analysis, we assume the kernels are real-valued.
\begin{theorem}\label{thm:generalization}
Assume $k_n^{\opn{poly},q}$ and $\hat{k}_n^{\opn{prod},q}$ are real-valued.
Let $D(k_n^{\opn{poly},q},x)=\sum_{j=1}^d\alpha_j\Vert R_n(x_{j})\Vert^{2q}_{\opn{op}}$ and $D(\hat{k}_n^{\opn{prod},q},x)=\prod_{j=1}^q(\Vert R_n(\tilde{k}_{1,j}(x,x))\Vert_{\opn{op}}\Vert R_n(\tilde{k}_{2,j}(x,x))\Vert_{\opn{op}})+\beta_n C$, 
%and $D(k_n^{\opn{sep},q},x)=\tilde{k}(x,x)\prod_{j=1}^q\Vert R_n(a_j)\Vert_{\opn{op}}^2$ for $x\in\alg_0^d$, 
where $\Vert\cdot\Vert_{\opn{op}}$ is the operator norm and $C=\prod_{j=1}^q\int_{\mathbb{T}}\tilde{k}_{1,j}(x(t),x(t))\mr{d}t\int_{\mathbb{T}}\tilde{k}_{2,j}(x(t),x(t))\mr{d}t$.
\red{Assume $\Vert {k}_n^{\opn{poly},q}(x,x)\Vert_{\alg} < \infty$ and $\Vert \hat{k}_n^{\opn{prod},q}(x,x)\Vert_{\alg}\le \infty$ for any $x\in\mcl{A}_0^d$.}
For $k_n=k_n^{\opn{poly},q},\hat{k}_n^{\opn{prod},q}$, \red{for any $z\in\mathbb{T}$,} and
for any $\delta\in (0,1)$, with probability at least $1-\delta$, \red{simultaneously for all $f\in\mathcal{F}$,} we have 
\begin{align}
\red{\mr{E}[g(f(X),Y)](z)
\le_{\alg} \frac1N\sum_{i=1}^Ng(f(X_i),Y_i)(z)
+2L\frac{B}{N}\bigg(\sum_{i=1}^ND(k_n,X_i)\bigg)^{1/2}
+3\sqrt{\frac{\log 1/\delta}{N}}}.\label{eq:generalization_bound}
\end{align} 
\end{theorem}
%
%\color{red}
Regarding the factor $D(k_n,x)$, we have the following lemma.
\begin{lemma}\label{lem:factor}
Assume $\beta_n\le \beta_{n+1}$ for $\hat{k}_n^{\opn{prod},q}$.
With the same assumptions as Theorem~\ref{thm:generalization}, for $k_n=k_n^{\opn{poly},q},\hat{k}_n^{\opn{prod},q}$ and for $x\in\alg_0^d$, we have $D(k_n,x)\le D(k_{n+1},x)$.
\end{lemma}
%\color{black}
%
%\color{blue}
\begin{remark}
As we stated just before Subsection~\ref{subsec:connection}, if a kernel in Example~\ref{ex:standard_kernels} is real-valued, then the corresponding kernel in Definition~\ref{def:kernels} is also real-valued.
Thus, Theorem~\ref{thm:generalization} is valid in this case.
\end{remark}

\begin{remark}
The factor $D(k_n,X_i)$ in Theorem~\ref{thm:generalization} is bounded by a constant independent of $N$ if $\Vert X_{i,j}\Vert_{\mathcal{A}}$ for $i=1,\ldots,N$, $j=1,\ldots,d$ is uniformly bounded or $\tilde{k}_{l,j}$ for $l=1,2$, $j=1,\ldots,q$ is uniformly bounded.
In this case, the second term of the right-hand side of Eq.~\eqref{eq:generalization_bound} is order $1/\sqrt{N}$.
More precisely, assume there exists a constant $E>0$ such that for any $i=1,\ldots,N$, $\Vert X_{i,j}\Vert_{\mathcal{A}}\le E$.
Then, we have
\begin{align*}
D(k_n^{\operatorname{poly}},X_{i})
&= \sum_{j=1}^d\alpha_j\Vert R_n(X_{i,j})\Vert_{\operatorname{op}}^{2q}
\le \sum_{j=1}^d\alpha_j\Vert Q_n^*M_{X_{i,j}}Q_n\Vert_{\operatorname{op}}^{2q}
\le \sum_{j=1}^d\alpha_j\Vert M_{X_{i,j}}\Vert_{\operatorname{op}}^{2q}\\
&\le \sum_{j=1}^d\alpha_j\Vert X_{i,j}\Vert_{\mathcal{A}}^{2q}
\le \sum_{j=1}^d\alpha_jE^{2q},
\end{align*}
where $Q_n=[e_1,\ldots,e_n]$ and $e_j(z)=\mr{e}^{\mr{i}jz}$ is the Fourier function.
In addition, assume there exists a constant $E>0$ such that for any $x\in \mathbb{C}$, $l=1,2$, and $j=1,\ldots,q$, $\Vert \tilde{k}_{l,j}(x,x)\Vert_{\mathcal{A}}\le E$.
Then, we have
\begin{align*}
D(\hat{k}_n^{\operatorname{prod}},X_i)
&=\prod_{j=1}^q(\Vert R_n(\tilde{k}_{1,j}(X_i,X_i))\Vert_{\opn{op}}\Vert R_n(\tilde{k}_{2,j}(X_i,X_i))\Vert_{\opn{op}})+\beta_n C\\
&\le \prod_{j=1}^q \Vert \tilde{k}_{1,j}(X_i,X_i)\Vert_{\mathcal{A}} \Vert \tilde{k}_{2,j}(X_i,X_i)\Vert_{\mathcal{A}}+\beta_n C
\le E^{2q}+\beta_n C.
\end{align*}
\end{remark}
%\color{black}

The first term of the right-hand side of~\eqref{eq:generalization_bound} is the empirical loss, and the second term represents the model complexity of the RKHM $\modu_{k_n}$.
There is a tradeoff between these terms.
If $n$ is larger, then since the approximation space $\opn{Span}\{e_{-(n-1)},\ldots,e_{n-1}\}$ in $\alg$ is larger, the outputs of ${k_n}$ can represent more functions, which enables us to describe the dependency of outputs on inputs more.
%Indeed, if $n=1$, then the output of $k_n$ is always a constant function, where the situation is the same as the existing separable kernel discussed in Section~\ref{sec:intro}.
Thus, the empirical loss (the first term) can become small if $n$ is large.
On the other hand, the complexity of $\modu_{k_n}$ (the second term of the right-hand side of~\eqref{eq:generalization_bound}) is larger if $n$ is larger.
%An advantage of introducing the parameter $n$ to construct the new kernels based on the typical kernels in Example~\ref{ex:standard_kernels} is that we can control the empirical loss and the complexity through $n$.
%By setting a proper $n$, we can balance the two terms and make the expected loss $\mr{E}[g(f(X),Y)]$, the left-hand side of Eq.~\eqref{eq:generalization_bound}, small.
\blue{The role of discretization in learning with functional data is not well understood in the literature and is often treated as a purely practical step. In our framework, discretization is performed through spectral truncation, rather than pointwise sampling. The truncation parameter $n$ controls the number of retained spectral components and thus the effective dimension of the hypothesis space.
Theorem~\ref{thm:generalization} provides a key insight into this question: it shows that $n$ acts as a regularization parameter, inducing a bias--variance tradeoff. In particular, spectral truncation offers a principled approach for kernel methods to discretization by aligning the representation with the spectral structure of the data.}

%\begin{proposition}\label{prop:rademacher_equiv}
%Let $\mcl{F}$ be an $\alg$-valued function class.
%Assume for any $x\in\alg_0^d$ and $z\in\mathbb{T}$, $f(x)(z)\in\mathbb{R}$.
%Then, we have $\hat{R}_N(\mathbf{x},\mcl{F}(z))=\hat{R}_{\alg,N}(\mathbf{x},\mcl{F})(z)$, where $\mcl{F}(z)=\{x\mapsto (f(x))(z)\,\mid\,f\in\mcl{F}\}$.
%\end{proposition}

\section{Kernel Ridge Regression with the Spectral Truncation Kernels}\label{sec:application}
We illustrate the effect of the proposed kernels by applying them to kernel ridge regression.
Let $x_1,\ldots,x_N\in\alg^d$ be input training samples and $y_1,\ldots,y_N\in\alg$ be output training samples.
We consider the case where we need the values of the output function evaluated at different $m$ points.
We consider the RKHM $\modu_k$ associated with $k=k_n^{\opn{poly},q}$ or $k=k_n^{\opn{prod},q}$.
Consider the typical minimization problem for regression, which is also considered by~\citet{hashimoto23_aistats}:
\begin{align}
\min_{f\in\modu_k}\bigg(\red{\frac1N}\sum_{i=1}^N\vert f(x_i)-y_i\vert_{\alg}^2+\lambda \vert f\vert_k^2\bigg),\label{eq:supervised}
\end{align}
where \red{$\lambda> 0$} is the regularization parameter.
We apply the approximation version of the representer theorem for RKHMs over general $C^*$-algebras (see Appendix~\ref{ap:representer}).
Then, the solution of the problem~\eqref{eq:supervised} is approximated by a vector having the form of $\sum_{i=1}^N\phi_i(x_i)c_i$.
The problem is reduced to
\begin{align}
&\min_{c_j\in\alg}\bigg(\red{\frac1N}\sum_{i=1}^N\bigg\vert \sum_{j=1}^Nk(x_i,x_j)c_j-y_i\bigg\vert_{{\alg}}^2+\lambda \bigg\vert \sum_{j=1}^N\phi(x_j)c_j\bigg\vert_{k}^2\bigg)\nn\\
&\quad=\min_{c_j\in\alg}(\red{1/N}(\bc^*\bG^2\bc-\bc^*\bG\by-\by^*\bG\bc\red{+\by^*\by})+\lambda \bc^*\bG\bc),\label{eq:min_prob}
\end{align}
where $\bG$ is the ${\alg^{N\times N}}$-valued Gram matrix whose $(i,j)$-entry is defined as $k(x_i,x_j){\in\alg}$, $\bc=[c_1,\ldots,c_N]^T\in\alg^N$, and $\by=[y_1,\ldots,y_N]^T\in\alg^N$.
%, and $\vert a\vert_{\alg}=(a^*a)^{1/2}$ for $a\in\alg$.
Then, the solution $\bc$ of the problem~\eqref{eq:min_prob} satisfies $\by=(\bG+\red{N}\lambda I)\bc$, which means $\by(z)=(\bG(z)+\red{N}\lambda I)\bc(z)$.
Therefore, for each $z\in\mathbb{T}$, we obtain $\bc(z)$ by computing $(\bG(z)+\red{N}\lambda I)^{-1}\by(z)$.

%\begin{remark}
\paragraph{Computational cost}
The computational cost for the product of two $n$ by $n$ Toeplitz matrices is $O(n^2)$, and computing $S_n(A)(z)$ for an $n$ by $n$ matrix $A$ and $z\in\mathbb{T}$ costs $O(n^2)$.
Thus, the cost for constructing $\bG(z_1),\ldots,\bG(z_m)$ for different $m$ points $z_1,\ldots,z_m\in\mathbb{T}$ with $k_n^{\opn{poly},q}$ or $\hat{k}_n^{\opn{prod},q}$ is $O((qn^2+mn^2)N^2)$. %, and $O(qn^2+mn^2+N^2)$ with $k_n^{\opn{sep},q}$.
As for the computation of $\bc(z)$, we need to compute $(\bG(z)+\red{N}\lambda I)^{-1}\by(z)$, which costs $O(N^3)$.
Thus, the total computational cost for obtaining $\bc(z_1),\ldots,\bc(z_m)$ is $O((q+m)n^2N^2+mN^3)$.
%On the other hand, that for $k^{\opn{poly},q}$, $\hat{k}^{\opn{prod},q}$, and $k^{\opn{sep},q}$ is $O(qmN^2+mN^3)$.
{Note that the cost is linear with respect to $m$. 
This is by virtue of considering function-valued kernels and RKHMs.
The situation is different if we use vvRKHSs.
For vvRKHSs, to obtain $m$ values as an output, we typically apply a $\mathbb{C}^{m\times m}$-valued kernel. 
Then, the Gram matrix should be $mN$ by $mN$.
Specifically, if the kernel is transformable, the computational cost is $O(m^3N^3)$.
Thus, if $(q+m)n^2<m^3N$, then the proposed kernels are more computationally efficient than the transformable kernels.
%See Appendices~\ref{ap:operator-valued_kernels} and \ref{ap:proposed_kernel} for more details.
For the proposed kernel, we can reduce the cost with respect to $N$, e.g., by applying Nystr\"{o}m method~\citep{drineas05} to each $\bG(z_1),\ldots,\bG(z_m)$.
Regarding $n$ and $q$, one approach to reducing the cost with respect to them is investigating the method to approximate the Fej\'{e}r kernel.
Generalization of the proposed kernel with the Fej\'{e}r kernel to a kernel with another integral kernel, which we will discuss in Subsection~\ref{subsec:generalization_heat} is also a way to reduce computational cost.
However, the main contribution of this paper is to show the advantage of introducing the noncommutivity in kernels, and investigating computationally effective methods is future work.}
%This is because the output of the kernel is a function by virtue of the map $S_n$, and we can separate the problem into $m$ scalar-valued problems.}
%\end{remark}

%\color{teal}
\section{Extension to combined model}\label{sec:deep}
%\color{black}
\green{This section is for readers interested in a more flexible, learnable variant. }
As we discussed in Subsection~\ref{subsec:connection}, the truncation parameter $n$ controls the global and local dependencies of the output on the input.
To make the framework more flexible and capture both global and local dependencies at the same time, we consider the following \blue{combined} model $f$ with $L$ spectral truncation kernels $k_{n_1}^1,\ldots, k_{n_L}^L$ and learnable parameters $c_i^1,\ldots,c_i^L\!\!\in\!\alg$:
\begin{align}
f(x)=\prod_{j=1}^L\bigg(\sum_{i=1}^{N}k_{n_j}^j(x,x_{i})c^j_{i}\bigg).\label{eq:deep}
\end{align}
If we set a different value of $n_j$ for each layer $j$, then we can capture different level of global and local dependencies for each layer, which enables us to capture both global and local dependencies.
%According to the corollary~\ref{cor:representer}, the coefficient $c_i$ in the approximation $\sum_{i=1}^N\phi(x_i)c_i$ of the solution is written as a finite sum of functions in $\alg$.

To learn the \blue{combined} model, we need to represent $c_i^j$ in Eq.~\eqref{eq:deep} as a finite sum of functions in $\alg$.
The following proposition provides an appropriate way to parameterize $c_i^j$ in the sense that the representation power of the model grows exponentially.
For simplicity, we consider the case where $n_j=n$ for $j=1,\ldots,L$ in the following proposition.
However, we can derive the same result for the case of $n_j\neq n_l$ for $j\neq l$.
\begin{proposition}\label{prop:representation_power}
Let $k^j_n(x,x_i)=\sum_{l=-(n-1)}^{n-1}\tilde{d}_{i,l}^j\mr{e}^{\mr{i}lz}$ with $\tilde{d}_{i,l}^j\in\mathbb{C}$.
Assume $c_i^j$ is parameterized as $c_i^j(z)=\sum_{l=-(n-1)}^{n-1}d_{i,l}^je_{\tau_j,l}$ with $d_{i,l}^j\in\mathbb{C}$ and $\tau_j\in\mathbb{R}$.
Here, $e_{\tau_j,l}\in C(\mathbb{T})$ satisfies $e_{\tau_j,l}(z)=\mr{e}^{\mr{i}\tau_jlz}$ for $z\in S$ for \red{an interval} $S$ of $\mathbb{T}$ \red{such that $S\neq \mathbb{T}$}.
Then, for $z\in S$, $f$ in Eq.~\eqref{eq:deep} is represented as 
\begin{align*}
\sum_{l_1,\ldots,l_L=-(n-1)}^{n-1}\sum_{m_1,\ldots,m_L=-(n-1)}^{n-1}\sum_{i=1}^N\tilde{d}_{i,l_1}^1\cdots \tilde{d}_{i,l_L}^L{d}_{i,l_1}^1\cdots {d}_{i,l_L}^L\mr{e}^{\mr{i}(l_1+\cdots +l_L+\tau_1m_1+\cdots \tau_Lm_L)z}.
\end{align*}
\end{proposition}
Let $V_{n,L}=\opn{Span}\{z\mapsto\mr{e}^{\mr{i}(l_1+\cdots +l_L)z}e_{\tau_1,m_1}(z)\cdots e_{\tau_L,m_L}(z)\,\mid\,l_1,\ldots,l_L,m_1,\ldots,m_L\in\{-(n-1),\ldots,n-1\}\}$.
If \red{$\{1,\tau_1,\ldots,\tau_L\}$} is linearly independent over $\mathbb{Z}$, then $V_{n,L}$ is described by \red{$2n-1+(2n-1)^L$} oscillating functions.
In this sense, the representation power of the model $f$ grows exponentially with respect to the number of layers $L$ \red{under the same assumption of Proposition~\ref{prop:representation_power} and linear independence of \red{$\{1,\tau_1,\ldots,\tau_L\}$} over $\mathbb{Z}$}.
\red{Note that the linear independence corresponds to rational independence to avoid frequency collisions.}
\begin{remark}
The exponential growth of the representation power is also observed for neural networks~\citep{hanin19}.
Deep learning with kernels has also been proposed~\citep{laforgue19,hashimoto23-deeprkhm}.
Whereas the growth is obtained by the composition for these existing frameworks, it is obtained by the product for our case.
Unlike the case of the existing frameworks, the representation power \red{for representing functions in} $\alg$ does not become high by the composition of functions.
\end{remark}

\begin{remark}
The function $f$ is in the RKHM associated with the kernel $\prod_{j=1}^Lk_{n_j}^j$.
Since the Rademacher complexity bound depends on the norm of the kernel, it can also grow exponentially with respect to $L$.
Thus, for the \blue{combined} setting, in the same manner as Theorem~\ref{thm:generalization}, we still have the tradeoff between the representation power and the model complexity.
In this case, it is controlled by both $n$ and $L$.
The exponential growth of the Rademacher complexity bound (thus, the generalization bound) with respect to the number of layers is also observed for neural networks~\citep{neyshabur15,bartlett17,golowich18}.
Investigating how we can take advantage of the \blue{combined} setting in the sense of the generalization is future work.
\end{remark}

\section{Applications: Extracting Local and Global Dependencies of functions}\label{sec:applications}
There are many potential applications of the spectral truncation kernels.
We list two examples here.

\paragraph{Time-series data analysis} 
We can regard a time-series as a function on a time space.
In many cases, a state at a certain time $z$ is influenced strongly by another state at the same time $z$, but also by the state around the time $z$.
Since commutative kernels focus only on local information, we cannot describe these two states with commutative kernels.
On the other hand, since separable kernels focus only on global information, we cannot describe the relationship of these two states at each time $z$.
By applying the proposed kernels, we can extract global information, but also can focus on local information.
\red{We remark that the proposed kernels basically work well for interpolation and nowcasting since they can represent bi-directional interactions over time.
If we construct kernels that can represent uni-directional interactions, it can be applied to time-series forecasting.
The construction of these kernels is future work.
}

\paragraph{Problems in computer vision}
\red{We can regard an image as a discretized function on the space of pixels.
Using the proposed kernels, we can encode relationships between different spatial locations or interactions across channels.
This can help in tasks such as image recovery, segmentation, texture analysis, or structured output prediction, where local and global visual information need to be balanced.
We will numerically confirm the applicability of the spectral truncation kernels to image recovery task in Subsection~\ref{subsec:mnist}.}

\paragraph{Operator learning}
\red{A major task of operator learning is} to obtain a solution of a partial differential equation as an output from an input function (such as initial condition or parameter of the equation).
Thus, we construct a model where both of the input and output are functions.
Applying kernel methods to operator learning has been proposed~\citep{batlle24}.
We can construct the model by solving a kernel ridge regression task.
Another well-known operator learning method is neural operator.
In the framework of neural operators, we apply integral operators to extract global information and apply local linear operators and local activation functions to extract local information.
The proposed kernel enables us to do similar procedures for the operator learning with kernels.
By considering the product of multiple proposed kernels with different values of $n$ or \blue{combined model} with the proposed kernels with different values of $n$, we can extract both global and local information; we can extract global information using the kernel with small $n$ and extract local information using the one with large $n$ in the model.
We will numerically confirm the applicability of the spectral truncation kernels to operator learning in Subsection~\ref{subsec:oplearn_exp}.

\section{Numerical results}\label{subsec:numexp}
We show numerical results that illustrate the effects of the noncommutativity on the learning process.
\blue{The experiments are to demonstrate that the proposed kernel (a) works in practice, (b) outperforms the baselines (separable and commutative kernels) it theoretically extends, and (c) the truncation parameter n plays the role predicted by theory.}

\blue{We used the formula in Definition~\ref{def:kernels} to compute the spectral truncation kernel.
To compute the $(j,l)$-entry of $R_n(x)$ for a function $x\in\mathcal{A}$, we need to compute the integral $\int x(t) \mathrm{e}^{-\mathrm{i}(j-l)t}\mathrm{d}t$, which is the Fourier transform of $x$.
We approximate the integral by the discrete sum $2\pi/T\sum_{r=1}^Tx(t_r)\mathrm{e}^{-\mathrm{i}(j-l)t_r}$ with $T$ discrete points $t_1,\ldots,t_T\in\mathbb{T}$.}
%For experimental details and additional results, see Appendix~\ref{sec:exp_detail}.

\subsection{Experiment with synthetic data}\label{subsec:exp_syn}
We first observed the behavior of the proposed kernels $k_n^{\opn{poly},q}$ and $\hat{k}_n^{\opn{prod},q}$ with different values of $n$ to observe that the truncation parameter $n$ controls global and local dependencies of the output on the input, as discussed in Subsection~\ref{subsec:connection}. %and compare them to the commutative kernels $k^{\opn{poly},q}$, $k^{\opn{prod},q}$, and $k^{\opn{sep},q}$.
We considered a regression task with synthetic data.
For $i=1,\ldots,N$, we generated input training samples as 
$x^i(z)=[\sin(0.01iz)+0.01\xi_{z,i},\cos(0.01iz)+0.01\eta_{z,i}]$.
We set the target function $f$ as $f=f_1,f_2$, where $f_1(x)(z)=(\sin(\cos(\int_{z-\Delta}^{z+\Delta}x_1(t)\mr{d}t+\int_{z-\Delta}^{z+\Delta}x_2(t)\mr{d}t)))$ \red{and $f_2(x)(z)=(\sin(\cos(\int_{z-5\Delta}^{z+5\Delta}x_1(t)\mr{d}t+\int_{z-5\Delta}^{z+5\Delta}x_2(t)\mr{d}t)))$} and $\Delta=2\pi/30$, and generated training output samples $y^i(z)=f(x)(z)+0.001\xi_{z,i}$.
Here, $\xi_{z,i}$ and $\eta_{z,i}$ were sampled independently from the Gaussian distribution with mean 0 and standard deviation 1.
We set the sample size $N$ as 1000.
For the kernel ${k}_n^{\opn{poly},q}$, we set $q=1$ and $\alpha_1=1$.
Note that ${k}^{\opn{poly},q}$ is the linear kernel in this case.
For the kernel $\hat{k}_n^{\opn{prod},q}$, we set $q=1$, $\tilde{k}_{1,1}(x,y)=\tilde{k}_{2,1}(x,y)=2\pi\mr{e}^{-\vert x-y\vert^2}$.
We set $\beta_n=1$ for $n<\infty$ and $\beta_{\infty}=0$.
%For the kernel $k_n^{\opn{sep},q}$, we set $\tilde{k}(x,y)=\mr{e}^{-0.1\cdot 2\pi^2(\Vert x_1-y_1\Vert_{L^2(\mathbb{T})}^2+\Vert x_2-y_2\Vert_{L^2(\mathbb{T})}^2)}$, $q=2$, and $a_1(z)=a_2(z)=2\pi\mr{e}^{\sin z}$.
We estimated $f$ using kernel ridge regression with the regularization parameter $\lambda$.
\blue{We applied the cross-validation grid search to find an optimal regularization parameter $\lambda$.}
We generated 1000 input test samples in the same manner as the input training samples, and evaluated the test error $(1/N)\sum_{i=1}^N \Vert \hat{f}(\tilde{x}^i)-\tilde{y}^i\Vert_{L^2(\mathbb{T})}^2$, where $\hat{f}$ is the estimation of $f$ and $\tilde{x}^i$ and $\tilde{y}^i$ are input and output test samples, respectively.
For $x^i$, $y^i$, and $\tilde{x}^i$, we discretized each function with 30 equally spaced points on $\mathbb{T}$.
\blue{Thus, the number $T$ of discretization points as 30.}
%with the regularization parameter $\lambda=0.01$. 
%For $x^i$, $y^i$, and $a$, we discretized each function with 30 equally spaced points on $\mathbb{T}$.
Figure~\ref{fig:separable} \red{(a--d)} illustrates the test error.
We can see that for ${k}_n^{\opn{poly},q}$, the test error is the smallest when $n=32$.
It becomes large if $n$ is larger or smaller than $32$.
This is because when $n=32$, the information of the global and that of the local dependencies are balanced, and the best possible error is obtained.
We can see that even for the simplest kernel ${k}_n^{\opn{poly},1}$, the proposed kernel ($n$ is finite) goes beyond the typical commutative kernel ($n=\infty$).
We have similar results for ${k}_n^{\opn{prod},q}$.
Note that $\hat{k}_{\infty}^{\opn{prod},q}(x,y)(z)=\hat{k}^{\opn{prod},q}(x,y)(z)$ only depends on $x(z)$ and $y(z)$.
However, $f(x)(z)$ depends on $x(t)$ for $t\in [z-\Delta,z+\Delta]$ \red{or $t\in [z-5\Delta,z+5\Delta]$}.
Thus, the test error becomes large when $n=\infty$.
In addition, as discussed in Remark~\ref{rmk:positive_definiteness}, $\hat{k}_n^{\opn{prod},q}$ may not be positive definite.
However, although we set $\beta_n$ as a small value, all the eigenvalues of the Gram matrix are positive in this case.
To investigate the positive definiteness of $\hat{k}_n^{\opn{prod},q}$, we computed the eigenvalues of the Gram matrix $\bG(0)$ at 0.
Figure~\ref{fig:separable} (e) shows the result.
The index $i$ is determined in the descendant order with respect to the value of $\lambda_i$.
Thus, we can see that although we set $\beta_n$ as a small value, all the eigenvalues of $\bG(0)$ are positive in this case.

We also compared the performance and  computational time with the proposed kernels to the existing operator-valued kernel with the same setting as above to show the advantage of the proposed kernel in the sense of the computational efficiency and performance.
For the existing kernel, we used the nonseparable kernel $k(x,y)=\tilde{k}(x,y)[\mr{e}^{-\vert x_i-y_j\vert^2}]_{i,j}$, where we replaced $S_n(\prod_{j=1}^q R_n(a_j)^*\prod_{j=1}^q R_n(a_j))$ with the nonseparable matrix $[\mr{e}^{-\vert x_i-y_j\vert^2}]_{i,j}$.
This kernel is the combination of the separable and transformable kernels and proposed by~\citet{lim15}.
Table~\ref{tab:regression} shows the result.
We can see that the proposed kernels outperform the existing typical nonseparable kernel, and the computational time with the proposed kernel is smaller than that with the existing kernel.
\green{We also documented additional results about combined models in Appendix~\ref{ap:additional_numexp}.}

\begin{figure}%\vspace{-.6cm}
%\begin{minipage}{0.55\textwidth}
\centering
\subfigure[$k={k}_n^{\opn{poly},q}$, $f=f_1$]{\includegraphics[scale=0.5]{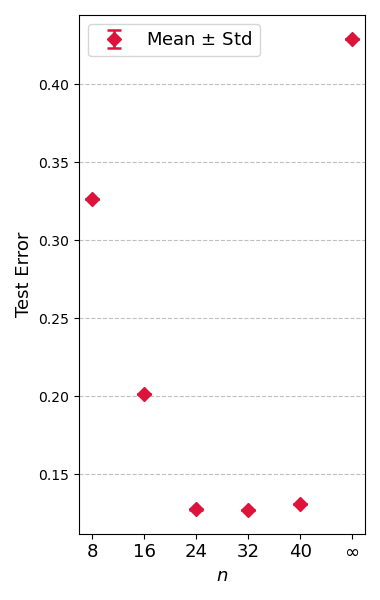}}
\subfigure[$k=\hat{k}_n^{\opn{prod},q}$, $f=f_1$]{\includegraphics[scale=0.5]{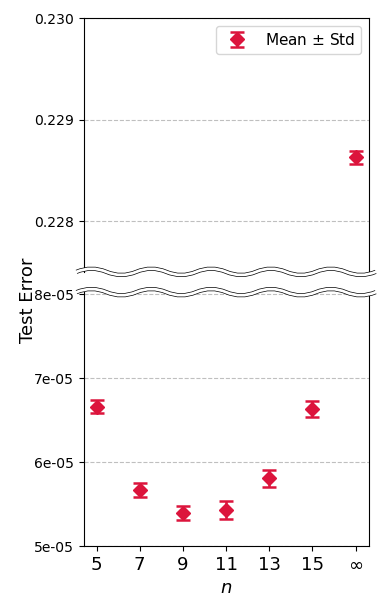}}
\subfigure[$k={k}_n^{\opn{poly},q}$, $f=f_2$]{\includegraphics[scale=0.5]{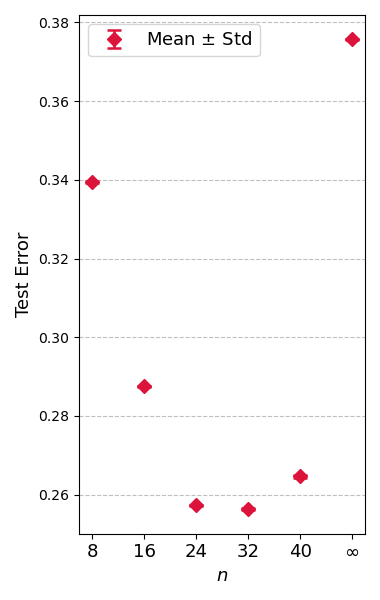}}
\subfigure[$k=\hat{k}_n^{\opn{prod},q}$, $f=f_2$]{\includegraphics[scale=0.5]{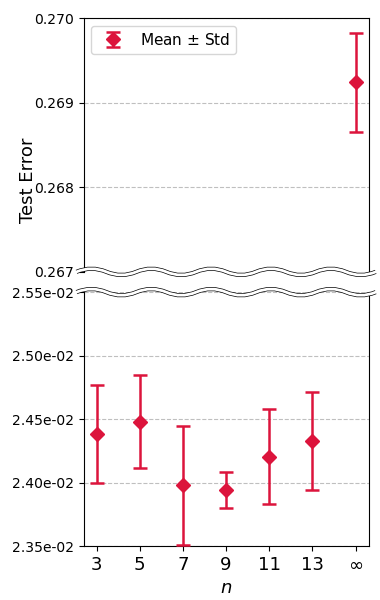}}
%\subfigure[$k=k_n^{\opn{sep},q}$]{\includegraphics[scale=0.45]{rebuttal/sep_cv.png}}
\subfigure[Eigenvalues, $f=f_1$]{\includegraphics[scale=0.25]{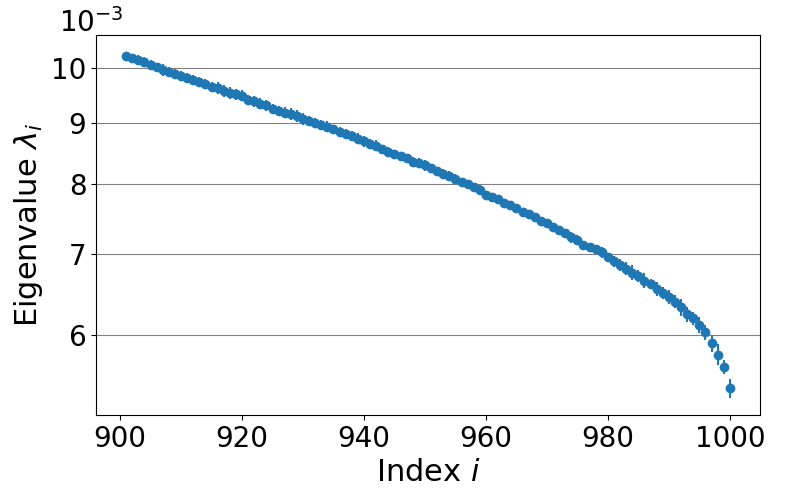}}
%\subfigure[Deep approach]{\includegraphics[scale=0.2]{deep_30.png}}
%\vspace{-.4cm}
\caption{(a--d) Test error of the regression task with different values of $n$. (Box plot of results of five independent runs with different values of noises $\xi_{z,i}$ and $\eta_{z,i}$ in $x^i(z)$ and $y^i(z)$.) (e) Eigenvalues of the Gram matrix $\bG(0)$ indexed as the descending order for the regression task. (Average value of results of five different runs. The error bar represents the standard deviation.)}\label{fig:separable}
%\end{minipage}\quad
%\begin{minipage}{0.4\textwidth}
 %   \centering
 %   \includegraphics[scale=0.25]{eig.png}
 %   \caption{Eigenvalues of the Gram matrix $\bG(0)$ indexed as the descending order. (Average value of results of five different runs. The error bar represents the standard deviation.)}
 %   \label{fig:eig}
%\end{minipage}
\end{figure}

\begin{table}[t]
\caption{Test error and CPU time of the regression task with RKHM and vvRKHS}\label{tab:regression}
\centering
\newcolumntype{C}{>{\centering\arraybackslash}X}
\begin{tabularx}{0.9\textwidth}{C|c|c}
    \hline
    &Test Error &  CPU Time (s)\\
    \hline
    {\small RKHM, $k=k_n^{\opn{prod},q}$ ($n=9$)} & {\small $5.40\times 10^{-5}\pm$7.79$\times 10^{-8}$} &{\small 149.3$\pm$2.392}\\
    %{\small RKHM, $k=k_n^{\opn{sep},q}$ ($n=10$)} & {\small 0.0385$\pm$5.04388$\times 10^{-5}$} & {\small 22.79$\pm$	0.5742}\\
    {\small vvRKHS (combination of transformable\newline  and separable kernels)} & {\small $3.70 \times 10^{-4} \pm 3.75\times 10^{-7}$} & {\small 570.4$\pm$14.87}\\
    \hline
\end{tabularx}
\end{table}

\begin{figure}[t]
\centering
\subfigure[]{\includegraphics[scale=0.5]{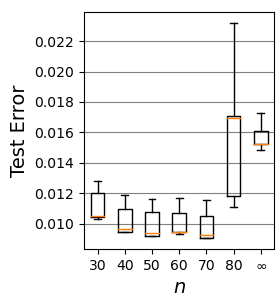}}
\subfigure[]{\includegraphics[scale=0.25]{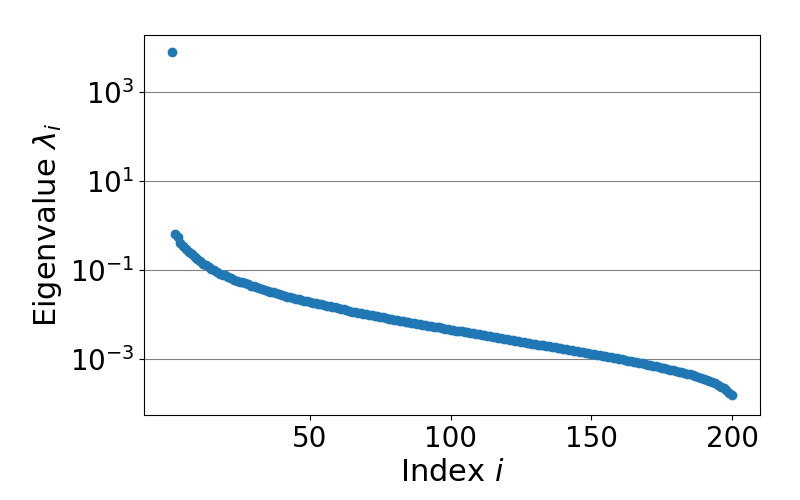}}
\subfigure[]{\includegraphics[scale=0.25]{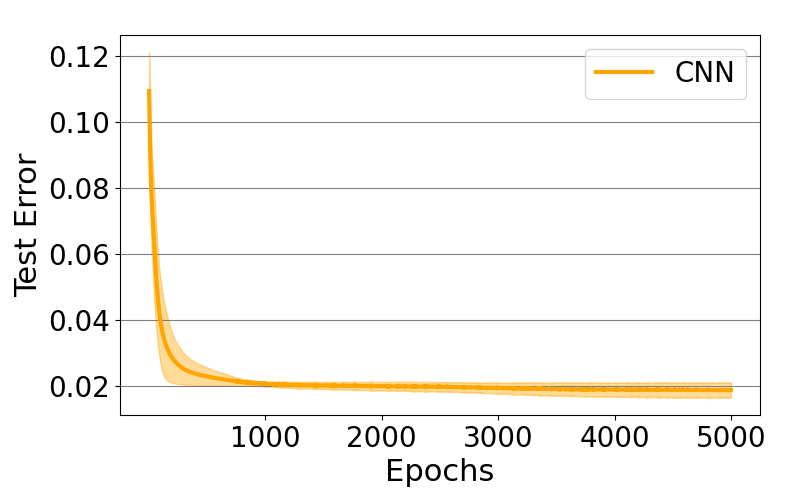}}
\caption{(a) Test error of the image recovering task with different values of $n$ and the kernel $k=\hat{k}_n^{\opn{prod},q}$. (Box plot of results of five independent runs with different training and test data.)  (b) Eigenvalues of the Gram matrix $\bG(0)$ indexed as the descending order for the image recovering task. \green{(c) Test error of the image recovering task with a 2-layer CNN.} (Average value of results of five different runs. The error bar represents the standard deviation.)}\label{fig:mnist_err}
\end{figure}

\begin{figure}[t]
    \centering
\tabcolsep = 0.1cm
\begin{tabular}{ccccccc}
\includegraphics[scale=0.9]{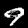} &
 \includegraphics[scale=0.9]{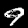} & 
 \includegraphics[scale=0.9]{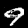}&
\includegraphics[scale=0.9]{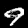}&
\includegraphics[scale=0.9]{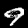}&
\includegraphics[scale=0.9]{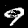}&
\includegraphics[scale=0.9]{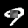}\\
$n=30$ & $n=40$ & $n=50$ & $n=60$ & $n=70$ & $n=80$ & $n=\infty$
\end{tabular}
    \caption{Output images of the image recovering task with different values of $n$.}
    \label{fig:mnist}
\end{figure}

\subsection{Experiment with MNIST}\label{subsec:mnist}
To check the performance of the proposed kernels for practical applications, we consider an image recovering task with MNIST~\citep{lecun98}.
For $i=1,\ldots,N$, we generated input training samples by setting all the $8\times 8$ pixels in the middle of the images as $0$.
We set the output samples as the original images.
We flatten the 2-dimensional $28\times 28$ image to a vector in $\mathbb{R}^{784}$ and regard it as a discretized function on the space of pixels.
\blue{Thus, the number $T$ of discretization points is $784$ in this case.}
We tried to estimate a function that transforms an image with the missing part into the original image.
We applied kernel ridge regression with the regularization parameter $N\lambda=0.01$. 
We set the sample size $N$ as $200$, and used the kernel $k_n^{\opn{prod},q}$ with $q=1$ and $\tilde{k}_{1,1}(x,y)=\tilde{k}_{2,1}(x,y)=2\pi\mr{e}^{-0.1\vert x-y\vert^2}$.
We set $\beta_n=0.01$ for $n<\infty$ and $\beta_{\infty}=0$.
We generated 200 input test samples in the same manner as the input training samples, and evaluated the test error $(1/N)\sum_{i=1}^N \Vert \hat{f}(\tilde{x}^i)-\tilde{y}^i\Vert_{L^2(\mathbb{T})}^2$, where $\hat{f}$ is the estimated map that maps the image with the missing part to its original image, and $\tilde{x}^i$ and $\tilde{y}^i$ are input and output test samples, respectively.
Figure~\ref{fig:mnist_err} (a) shows the test error.
We can see that the test error becomes the smallest when $n=70$, but it becomes large when $n$ is smaller or larger than $70$.
In addition, to investigate the positive definiteness of $\hat{k}_n^{\opn{prod},q}$, we computed the eigenvalues of the Gram matrix $\bG(0)$ at 0.
Figure~\ref{fig:mnist_err} (b) shows the result.
The index $i$ is determined in the descendant order with respect to the value of $\lambda_i$.
Thus, we can see that although we set $\beta_n$ as a small value, all the eigenvalues of $\bG(0)$ are positive in this case.
Furthermore, Figure~\ref{fig:mnist} shows the output images with different values of $n$.
When $n=\infty$, i.e., the commutative kernel, we cannot recover the missing part since each pixel of the output is determined only with the corresponding pixel of the input, and we cannot obtain information of other pixels to estimate what is written in the image.
When $n=70$, we can recover the image the clearest.

\green{In addition, to compare the results with those of other standard methods, we applied convolutional neural network (CNN) to the same task.
Since we considered the product of two kernels ($q=1$) for $\hat{k}_n^{\opn{prod},q}$, we constructed a CNN with two layers.
To extract local and global information simultaneously, we set a $3\times 3$ filter as the first layer (for local information) and a $24\times 24$ filter as the second layer (for global information).
For the optimization, we used the Adam optimizer with learning rate $0.001$ to minimize the mean squared error.
The result is shown in Figure~\ref{fig:mnist_err} (c).
We can see that the kernel ridge regression with the proposed spectral truncation kernel outperforms the CNN.
}

\subsection{Experiment for operator learning}\label{subsec:oplearn_exp}
To show the availability of the proposed kernels to operator learning, we conducted an experiment.
Consider the following Burgers' equation on $(0,1]\times (0,1]$ for $w(x,t)$.
\begin{align*}
\frac{\partial w}{\partial t}+w\frac{\partial w}{\partial x}&=\nu\frac{\partial^2 w}{\partial x^2}\qquad x\in (0,1)\times (0,1],\\
w(x,0)&=u(x)\qquad x\in (0,1),
\end{align*}
where $\nu=0.1$. 
We learned the operator mapping the initial condition $u$ to $w(\cdot,1)$, the solution at time $t=1$, by kernel ridge regression explained in Section~\ref{sec:application}.
The training data and test data are generated by sampling the initial condition $u$ from the Gaussian distribution with mean 0 and the covariance $625(-\Delta+25 I)^{-2}$.
We used equally spaced $128$ points for computing the integral related to the input functions.
The number of training samples is 1000 and that of test samples is 200.
The same problem is also considered by~\citet{batlle24}.
We used the same $\hat{k}_n^{\opn{prod},q}$ kernel as Subsection~\ref{subsec:exp_syn} with $n=5$.
\blue{We set the number $T$ of discretization points as 128.}
We set the regularization parameter $N\lambda$ as $0.01$.

Figure~\ref{fig:burger_sol} shows the true solution $w(\cdot,1)$ and the predicted function $v$ by using the proposed kernel for the same test input function $u$ (initial condition).
Figure~\ref{fig:burger_error} shows the error $v-w(\cdot,1)$ of these two functions.
The mean value of the pointwise error $(1/128)\sum_{i=1}^{128}\vert v(x_i)-w(x_i,1)\vert$ over all the test samples is $0.00157\pm 7.079\times 1e^{-6}$ (average $\pm$ standard deviation over five independent runs).
We can see that the proposed kernel combined with kernel ridge regression properly predict the solution of the problem, and compared to the results in Subsection 4.3.1 by~\citet{batlle24}, the proposed kernel outperforms these existing results.

\begin{figure}[t]
    \centering
    \includegraphics[scale=0.4]{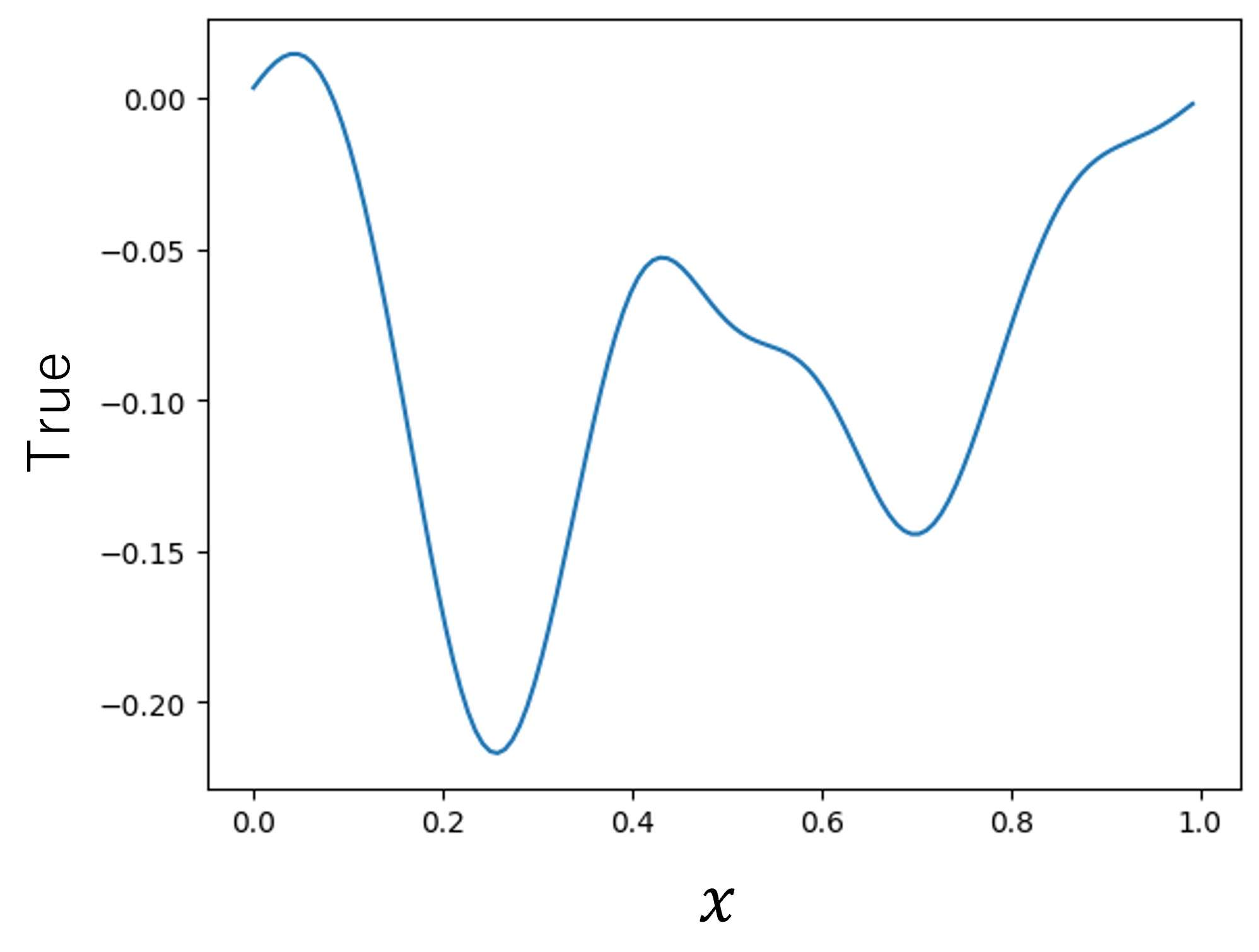}
    \includegraphics[scale=0.4]{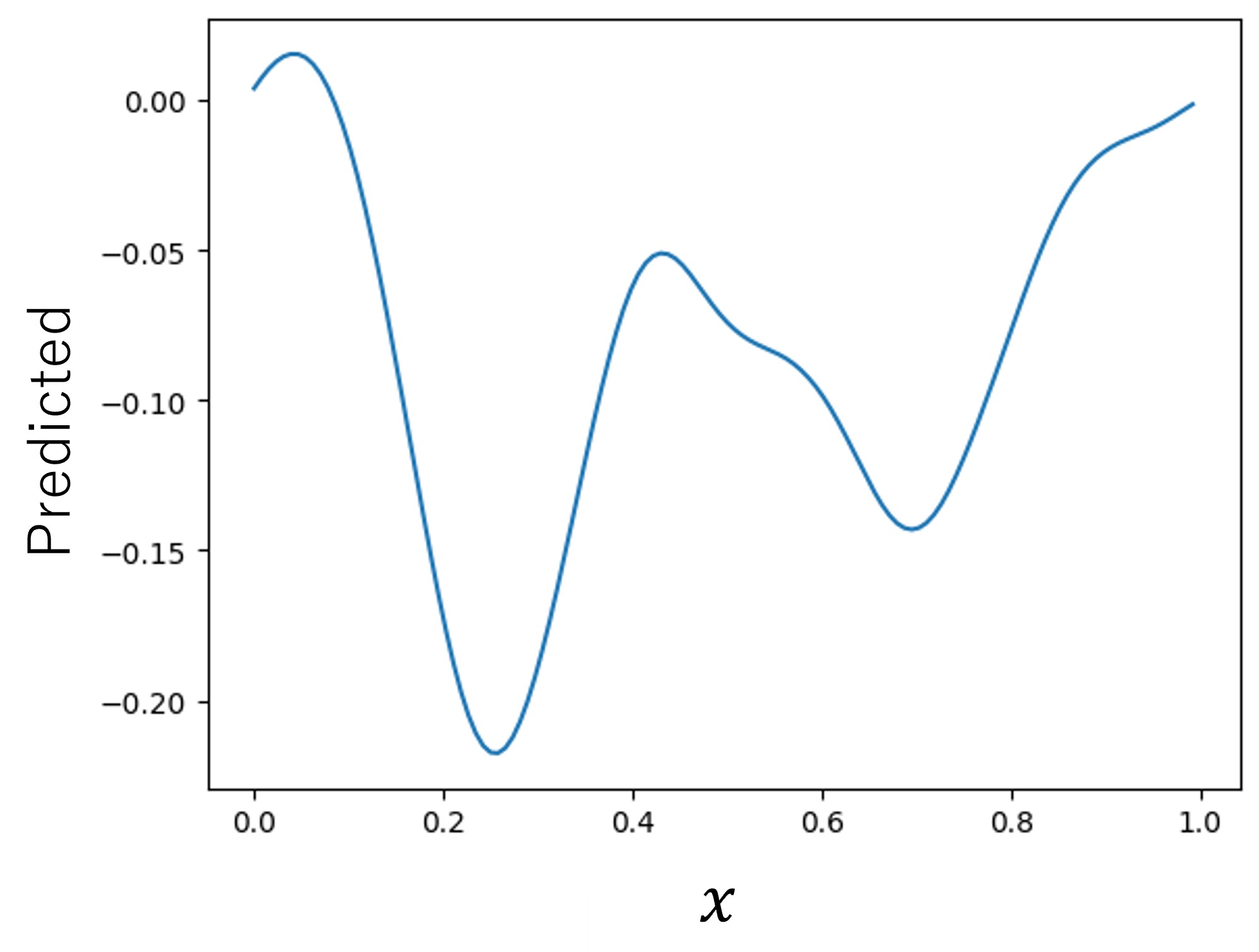}
    \caption{The true solution $w(\cdot,1)$ and the predicted function $v$ of the Burgers' equation.}
    \label{fig:burger_sol}
\end{figure}

\begin{figure}[t]
    \centering
    \includegraphics[scale=0.45]{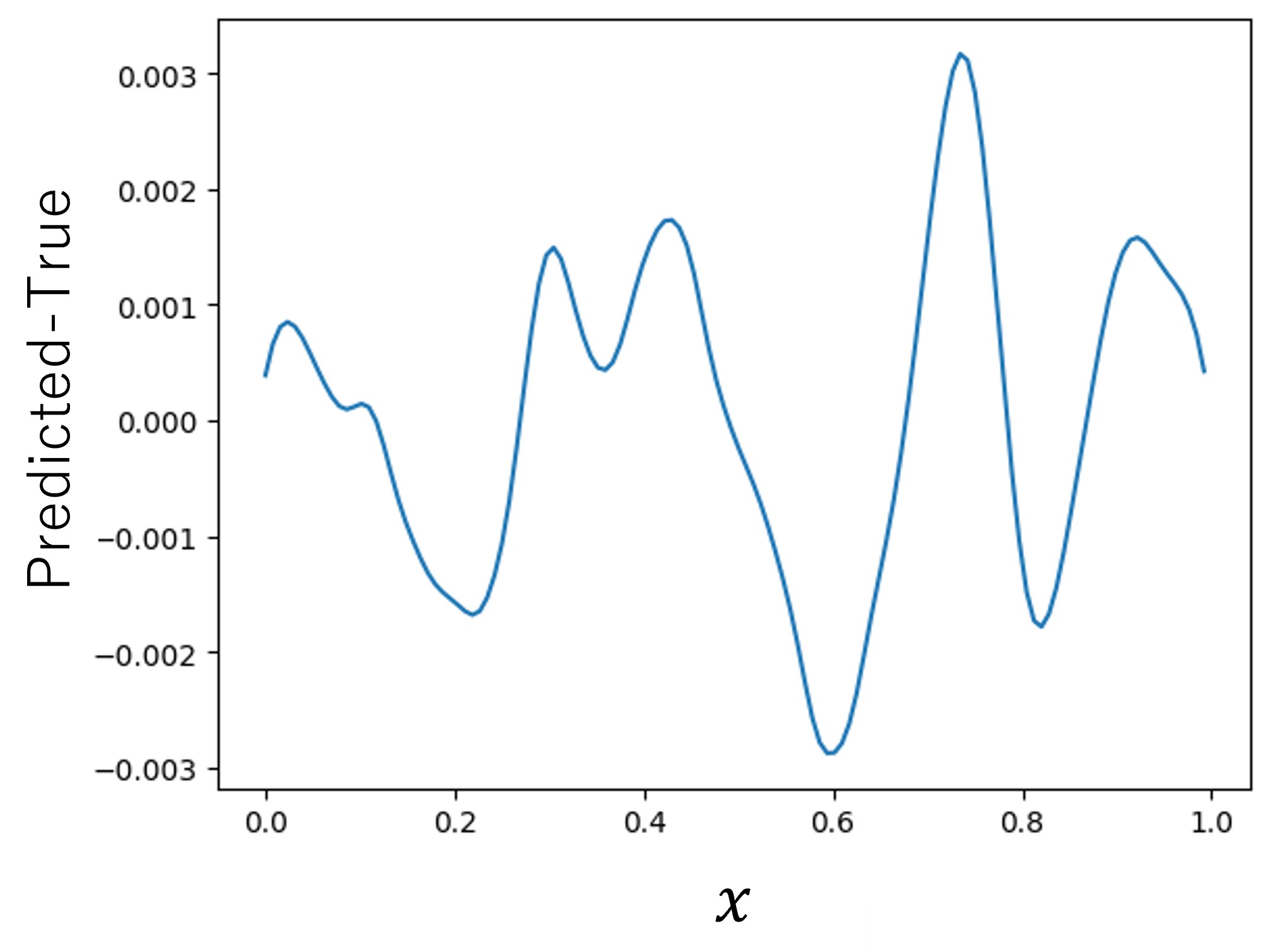}
    \caption{Pointwise error $v-w(\cdot,1)$ of the true solution $w(\cdot,1)$ and the predicted function $v$ of the Burgers' equation.}
    \label{fig:burger_error}
\end{figure}

%\begin{figure}
%\centering
%\begin{minipage}{0.2\textwidth}
%\includegraphics[scale=0.5]{mnist.png}
%\caption{Test error of the image recovering task with different values of $n$. (Box plot of results of five independent runs with different training and test data}
%\end{minipage}\qquad\qquad
%\begin{minipage}{0.75\textwidth}
%\tabcolsep = 0.1cm
%\begin{tabular}{ccccccc}
%\includegraphics[scale=0.9]{result30_400_26.png} &
% \includegraphics[scale=0.9]{result40_400_26.png} & 
% \includegraphics[scale=0.9]{result50_400_26.png}&
%\includegraphics[scale=0.9]{result60_400_26.png}&
%\includegraphics[scale=0.9]{result70_400_26.png}&
%\includegraphics[scale=0.9]{result80_400_26.png}&
%\includegraphics[scale=0.9]{resultinf_400_26.png}\\
%$n=30$ & $n=40$ & $n=50$ & $n=60$ & $n=70$ & $n=80$ & $n=\infty$
%\end{tabular}
%\end{minipage}
%\caption{(Left) Test error of the image recovering task with different values of $n$. (Box plot of results of five independent runs with different training and test data.)\quad (Right) Output image of the image recovering task with different values of $n$.}\label{fig:mnist}

%\end{figure}
\section{Generalization of the Spectral Truncation Kernels}\label{sec:kernel_generalization}
The proposed spectral truncation kernels have potential possibilities of constructing kernels for more general situations.

%\color{blue}
\subsection{Generalization to continuous functions on high-dimensional torus}
We can generalize Theorem~\ref{thm:convergence} for $\mathbb{T}$ to that for $\mathbb{T}^m$ in a straightforward manner.
%Let $e_j$ be the Fourier function defined as $e_j(z)=\mr{e}^{\mr{i}j\cdot z}$ for $j\in\mathbb{Z}^m$ and $z\in\mathbb{T}^m$.
Let $N=\{0,\ldots,n-1\}^m$.
Then, let $R_n(x)\in\mathbb{C}^{n^m\times n^m}$ be defined as
\begin{align*}
R_n(x)_{j,l}=
\int_{\mathbb{T}^m}x(t)\mr{e}^{-\mr{i}(j-l)\cdot t}\mr{d}t.
\end{align*}
Here, the integral is with respect to the normalized Haar measure on $\mathbb{T}^m$.
For a matrix $A\in\mathbb{C}^{n^m\times n^m}$, let $S_n(A)\in C(\mathbb{T}^m)$ be the function defined as $S_n(A)(z)=(1/n)\sum_{j,l\in N}A_{j,l}\mr{e}^{\mr{i}(j-l)\cdot z}$, where $A_{j,l}$ is the $(j,l)$-entry of $A$.
We define $k_n^{\opn{poly},q}$ and $k_n^{\opn{prod},q}$ in the same manner as for $\mathbb{T}$.
Then, we have the same result as Theorem~\ref{thm:convergence}.

\begin{theorem}\label{thm:convergence_Td}
For $x,y\in\alg^d$ and $z\in\mathbb{T}^m$, $k_n^{\opn{poly},q}(x,y)(z)\to k^{\opn{poly},q}(x,y)(z)$ and $k_n^{\opn{prod},q}(x,y)(z)\to k^{\opn{prod},q}(x,y)(z)$ as $n\to\infty$.
\end{theorem}

\subsection{Generalization to continuous functions on a compact space in $\mathbb{R}^m$}\label{subsec:generalization_compact}
Furthermore, we can generalize $\mathbb{T}^m$ to a compact space $\mathcal{X}\subset\mathbb{R}^m$ that is isomorphic to the unit ball.
Indeed, let $B_m$ be the closed unit ball in $\mathbb{T}^m$.
If $\mathcal{X}$ is isomorphic to $B_m$, then we can construct a function $\check{f}$ on $\mathbb{T}^m$ that satisfies $\check{f}(x)=\tilde{f}(x)$ for $x\in B_m$, where $\tilde{f}=f\circ \psi$ and
%Indeed, let $B_r=\{x\in\mathbb{R}^r\,\mid\,\Vert x\Vert\le 1\}$ be the unit ball.
$\psi:B_m\to \mathcal{X}$ is the isomorphism.
Let $a$ be a positive real number satisfying $1<a<\pi$.
Then, we can smoothly extend $\tilde{f}$ on $B_r$ to a map $\hat{f}$ on $aB_m$ as $\hat{f}(x)=\tilde{f}(x)\ (x\in B_m)$, $\hat{f}(x)=0\ (\Vert x\Vert=a)$.
For example, we can construct $\hat{f}$ in the same manner as a smooth bump function~\citep{tu08}.
Finally, we extend $\hat{f}$ on $aB_m$ to a map $\check{f}$ on $[-\pi,\pi]^m$ as $\check{f}(x)=\hat{f}(x)\ (x\in aB_m)$, $\check{f}(x)=0\ (x\notin aB_m)$.
Then, since $\check{f}([-\pi,\ldots,-\pi])=\check{f}([\pi,\ldots,\pi])$, we can regard $\check{f}$ as a 
function on $\mathbb{T}^m$.

\subsection{Generalization to multivariate functions}
We can generalize univariate functions to multivariate functions.
Indeed, let $r\in\mathbb{N}$.
Then, for $(f_1,\ldots,f_r)\in C(\mathbb{T}^d)^r$, we set $f\in C(\mathbb{T}^{d+1})$ as $f(x,i/r)=f_i(x)$.
Thus, we can set $\mathcal{A}$ as $C(\mathbb{T}^{m+1})$ instead of $C(\mathbb{T}^m)$ to deal with $r$-variate functions.
In this case, we can show the same statement as Theorem~\ref{thm:convergence_Td}  through the isomorphism $\psi$ in Subsection~\ref{subsec:generalization_compact}.

%\color{black}
\subsection{Generalization to other $C^*$-algebras}\label{sec:general_c_star}
In the framework of the spectral truncation, setting $R_n$ and $S_n$ for more general $C^*$-algebras has been investigated.
Using these results and replacing $R_n$ and $S_n$ in Definition~\ref{def:kernels}, we can define positive definite kernels for the $\alg^d$-valued inputs and $\alg$-valued output for a more general $C^*$-algebra $\alg$.
\paragraph{Continuous functions on high-dimensional torus~\citep{leimbach24}}
For $\alg=C(\mathbb{T}^m)$, let $e_j(z)=\mr{e}^{\mr{i}j\cdot z}$ for $j\in\mathbb{Z}^m$ and $B_n=\{j\in\mathbb{Z}^m\,\mid\,\Vert j\Vert\le n\}$, and consider the space $\opn{Span}\{e_j\,\mid\,j\in B_n\}$.
Here, $\Vert\cdot\Vert$ is the Euclidean norm.
We consider generalized matrices whose elements are indexed by $\mathbb{Z}^m$, and set 
$R_n(x)=(\int_{\mathbb{T}^m}x(t)\mr{e}^{-\mr{i}(j-l)\cdot t}\mr{d}t)_{j,l\in B_n}$ and $S_n(A)(z)=(1/\vert B_n\vert)\sum_{j,l\in B_n}A_{j,l}\mr{e}^{\mr{i}(j-l)\red{\cdot}z}$.

\paragraph{Continuous functions on the sphere~\citep{rieffel04}}
For $\alg=C(S^2)$, let $\rho_n$ be the $n$-dimensional irreducible group representation of $SU(2)$ (special unitary group of degree 2) and let $V_n$ be the representation space.
%Let $P\in\mcl{B}(V_n)$ be a projection, where $\mcl{B}(V_n)$ is the space of bounded linear operators on $V_n$.
%We set $R_n(x)=n\int_{SU(2)}x(g)\rho_n(g)^*P\rho_n(g)\mr{d}g$ and $S_n(A)(g)=\opn{tr}(A\rho_n(g)^*P\rho_n(g))$.
\red{Let $ S_n(A)([g])= \opn{tr}(A \rho_n(g)^\ast P \rho_n(g) ) $  where $[g]$ is the class of $g\in SU(2)$ in $ SU(2)/U(1)\cong S^2 $, $U(1)$ is the unitary group of degree 1, $P=\xi\xi^*$ is the projection, and $\xi$ is the highest weight vector (see \citep{rieffel04} for more details).
Let $R_n(x)=n\int_{SU(2)}  x([g]) \rho_n(g)^\ast P \rho_n(g) \mr{d}g  $, where the integral is with respect to the normalized Haar measure on $SU(2)$.}

\subsection{Generalization using other integral kernels}\label{subsec:generalization_heat}
We can generalize the spectral truncation kernel to a kernel on a measure space $\mcl{X}$ by replacing the function $F_n^{2q,P}$ in Eq.~\eqref{eq:kernel_fj} with another function $G
:\Theta \times \mcl{X}^{2q}$ that satisfies $G(\theta,z)\ge 0$ for any $\theta\in\Theta$ and $z\in\mcl{X}^{2q}$.
Here, $\Theta$ is the space of parameters.
For example, we can set $\Theta=\mathbb{R}_+$, the set of positive real values, $\mcl{X}=\mathbb{R}^d$, and $G(\sigma,z)=\mr{e}^{-\sigma \Vert z\Vert^2}$.
Then, similar to the spectral truncation kernel, the convolution with $G(\sigma,z-\cdot)$ captures global information if $\sigma$ is small and captures local information if $\sigma$ is large.
In this case, since the parameter space $\Theta$ is continuous, we can use an optimization method to learn a best possible parameter $\sigma$ from data.

\section{Conclusion and Discussion}\label{sec:conclusion}
In this paper, we proposed a new class of positive definite kernels for vector- or function-valued outputs based on the spectral truncation, which fills the gap between existing separable and commutative kernels.
The proposed kernels with the framework of RKHMs resolve two shortcomings of the framework of vvRKHSs for vector- or function-valued outputs at the same time: computational cost and choice of kernels.
By considering function-valued positive definite kernels, not operator-valued ones, we can alleviate the computational cost.
At the same time, we can introduce noncommutativity into the learning process and can induce interactions along the data function domain, which enables us to control the global and local dependencies of the output on the input.
%We also showed that we can control the tradeoff between the representation power and the model complexity associated with the proposed kernels.
%In addition, the proposed kernels are indexed by a truncation parameter $n$, and they converge to typical kernels: polynomial, product, and separable kernels, as $n$ goes to infinity.
%Setting $n$ as a finite value, {we can introduce noncommutativity into the learning process and can induce interactions along the data function domain, which enables us to control the tradeoff between the representation power and the model complexity of the RKHMs associated with the proposed kernels.
%These advantages make the proposed kernel more powerful than the existing kernels.
%Since $n$ describes the resolution of the discretization functions, our results show the positive effect of discretization on generalization.
In addition, we proposed a \blue{combined model} to obtain models with more flexibility.
%\todo[inline]{Our results show the positive effect of discretization on generalization}
%We applied the proposed kernel to kernel ridge regression and numerically observed that the proposed kernels outperformed the existing kernels.

In the current setting, we focus on the $C^*$-algebra of continuous functions on the torus.
Although this setting includes many important examples, e.g., periodic time-series data, continuous functions defined on a compact set in $\mathbb{R}$, finite-dimensional vectors, investigations for more general $C^*$-algebras allow us to apply the proposed kernel more general settings.
As we discuss in Section~\ref{sec:kernel_generalization}, we can generalize this setting and define positive definite kernels for more general inputs and outputs and for more flexible frameworks.
More detailed investigations for general settings are future work.
In addition, as we discuss in Section~\ref{sec:applications}, there are many potential applications of the proposed kernels.
\green{Evaluation on larger problems is an interesting direction of future work.}
Theoretical and empirical investigations for these generalized settings and applications remain to be investigated in future work.

\subsection*{Acknowledgements}
The representer theorem stated in Appendix~\ref{ap:representer} is pointed out by Dr. Miho Mukohara.
We would like to thank her for the helpful discussion.
We would also like to thank Dr. Noboru Isobe for a constructive discussion.
Ayoub Hafid was partially supported by FoPM, WINGS Program, the University of Tokyo.
Masahiro Ikeda was partially supported by grant JPMJCR1913 from JST CREST and 24K21316, 25K24910 from JSPS KAKENHI.
Hachem Kadri was partially supported by grant ANR-19-CE23-0011 from the French National Research Agency.

\appendix
\section*{Appendix}
\allowdisplaybreaks

\section{Proofs}\label{ap:proof}
We provide the proofs of theorems, propositions, and lemmas stated in the main text.
\begin{mythm}[Theorem~\ref{thm:convergence}]
For $x,y\in\alg^d$ and $z\in\mathbb{T}$, $k_n^{\opn{poly},q}(x,y)(z)\to k^{\opn{poly},q}(x,y)(z)$ and $k_n^{\opn{prod},q}(x,y)(z)\to k^{\opn{prod},q}(x,y)(z)$ as $n\to\infty$.
\end{mythm}
%To show Theorem~\ref{thm:convergence}, we first show the following lemma.
\begin{proof}
We show $k_n^{\opn{poly},q}(x,y)(z)\to k^{\opn{poly},q}(x,y)(z)$.
The proof for $k_n^{\opn{prod},q}(x,y)$ is similar.
We have
\begin{align*}
&k_n^{\opn{poly},q}(x,y)(z)
=S_n\bigg(\sum_{i=1}^d\alpha_i(R_n(x_i)^*)^qR_n(y_i)^q\bigg)(z)\\
&=\frac1{n}\sum_{i=1}^d\alpha_i\sum_{j,l=0}^{n-1}\sum_{r_1,\ldots,\red{r_{2q-1}=0}}^{n-1}R_n(x_i)^*_{j,r_1}\cdots R_n(x_i)^*_{r_{q-1},r_q} R_n(y_i)_{r_q,r_{q+1}}\cdots R_n(y_i)_{r_{2q-1},l}\mr{e}^{\mr{i}(j-l)z}\\
&=\frac1{n}\sum_{i=1}^d\alpha_i\sum_{j,l=0}^{n-1}\sum_{r_1,\ldots,r_{2q-1}=0}^{n-1}\int_{\mathbb{T}}\cdots\int_{\mathbb{T}}\overline{x_i(t_1)}\cdots \overline{x_i(t_q)}y_i(t_{q+1})\cdots y_i(t_{2q})\\
&\qquad\cdot \mr{e}^{\mr{i}(r_1-j)t_1}\cdots\mr{e}^{\mr{i}(r_{q}-r_{q-1})t_q}\mr{e}^{\mr{i}(r_{q+1}-r_{q})t_{q+1}}\cdots\mr{e}^{\mr{i}(l-r_{2q-1})t_{2q}}\mr{d}t_1\cdots\mr{d}t_{2q}\mr{e}^{\mr{i}(j-l)z}\\
&=\frac1{n}\sum_{i=1}^d\alpha_i\sum_{j,l=0}^{n-1}\sum_{r_1,\ldots,r_{2q-1}=0}^{n-1}\int_{\mathbb{T}}\cdots\int_{\mathbb{T}}\overline{x_i(t_1)}\cdots \overline{x_i(t_q)}y_i(t_{q+1})\cdots y_i(t_{2q})\\
&\qquad\cdot \mr{e}^{\mr{i}(r_1-j)(t_1-z)}\cdots\mr{e}^{\mr{i}(r_{q}-r_{q-1})(t_q-z)}\mr{e}^{\mr{i}(r_{q+1}-r_{q})(t_{q+1}-z)}\cdots\mr{e}^{\mr{i}(l-r_{2q-1})(t_{2q}-z)}\mr{d}t_1\cdots\mr{d}t_{2q}\\
&=\frac1{n}\sum_{i=1}^d\alpha_i\sum_{m=0}^{n-1}\bigg(\sum_{j,l,r_1,\ldots,r_{2q-1}=0}^m-\sum_{j,l,r_1,\ldots,r_{2q-1}=0}^{m-1}\bigg)\int_{\mathbb{T}}\cdots\int_{\mathbb{T}}\overline{x_i(t_1)}\cdots \overline{x_i(t_q)}y_i(t_{q+1})\cdots y_i(t_{2q})\\
&\qquad\cdot \mr{e}^{\mr{i}(r_1-j)(t_1-z)}\cdots\mr{e}^{\mr{i}(r_{q}-r_{q-1})(t_q-z)}\mr{e}^{\mr{i}(r_{q+1}-r_{q})(t_{q+1}-z)}\cdots\mr{e}^{\mr{i}(l-r_{2q-1})(t_{2q}-z)}\mr{d}t_1\cdots\mr{d}t_{2q}\\
&=\frac1{n}\sum_{i=1}^d\alpha_i\sum_{m=0}^{n-1}\sum_{\substack{j\lor l\lor r_1\lor\ldots\lor r_{2q-1}= m\\ 0\le j,l,r_1,\ldots,r_{2q-1}\le m}}\int_{\mathbb{T}}\cdots\int_{\mathbb{T}}\overline{x_i(t_1)}\cdots \overline{x_i(t_q)}y_i(t_{q+1})\cdots y_i(t_{2q})\\
&\qquad\cdot \mr{e}^{\mr{i}(r_1-j)(t_1-z)}\cdots\mr{e}^{\mr{i}(r_{q}-r_{q-1})(t_q-z)}\mr{e}^{\mr{i}(r_{q+1}-r_{q})(t_{q+1}-z)}\cdots\mr{e}^{\mr{i}(l-r_{2q-1})(t_{2q}-z)}\mr{d}t_1\cdots\mr{d}t_{2q}\\
&=\frac1{n}\sum_{i=1}^d\alpha_i\sum_{j=0}^{n-1}\sum_{r\in jP\bigcap \mathbb{Z}^{2q}}\int_{\mathbb{T}}\cdots\int_{\mathbb{T}}\overline{x_i(t_1)}\cdots \overline{x_i(t_q)}y_i(t_{q+1})\cdots y_i(t_{2q})\\
&\qquad\cdot \mr{e}^{\mr{i}r_1(z-t_1)}\cdots\mr{e}^{\mr{i}r_{q}(z-t_q)}\mr{e}^{\mr{i}r_{q+1}(z-t_{q+1})}\cdots\mr{e}^{\mr{i}r_{2q}(z-t_{2q})}\mr{d}t_1\cdots\mr{d}t_{2q}\\
&=\sum_{i=1}^d\alpha_ig_i\ast F_n^{2q,P}(z),
\end{align*}
where the sum $\sum_{j,l,r_1,\ldots,r_{2q-1}=0}^{-1}$ is set as $0$, $g_i(t)=\overline{x_i(t_1)}\cdots \overline{x_i(t_q)}y_i(t_{q+1})\cdots y_i(t_{2q})$ for $t=[t_1,\ldots,t_{2q}]\in\mathbb{T}^{2q}$, and  $P=\{r=[r_1,\ldots,r_{2q}]\in\mathbb{R}^{2q}\,\mid\,\vert \sum_{i=l}^mr_i\vert\le 1,\ l\le m\}$.
In addition, we set $j=r_0$ and \red{$l=r_{2q}$}.
The second to the last equality is derived by Lemma~\ref{lem:index} below.
Since $P$ is a convex polyhedron, $g_i\ast F_n^{2q,P}(z)\to g_i(z)$ as $n\to\infty$ by Lemma~\ref{lem:fejer_convergence_general}.
%For $k_n^{\opn{sep},q}$, we additionally use the fact $\lim_{n\to\infty}\tilde{k}(S_n(R_n(x)),S_n(R_n(y)))=\tilde{k}(\lim_{n\to\infty}S_n(R_n(x)),\lim_{n\to\infty}S_n(R_n(y)))$, which follows from the continuity of $\tilde{k}$, and the fact $S_n(R_n(x))\to x$ uniformly as $n\to\infty$ if $x$ and $y$ are differentiable~\cite[Lemma 10]{suijlekom21}.
\end{proof}

\begin{mythm}[Lemma~\ref{lem:index}]
For $m\in\mathbb{N}$ and $j=0,\ldots 2q$, let $Q_j^m=\{r'=[r_1',\ldots,r_{2q}']\in\mathbb{R}^{2q}\,\mid\, r_i'=r_i-r_{i-1}\ (i=1,\ldots,2q),\ 0\le r_i\le m\ (i=0,\ldots,2q),\ r_j=m\}$ and $P=\{r=[r_1,\ldots,r_{2q}]\in\mathbb{R}^{2q}\,\mid\,\vert \sum_{i=l}^kr_i\vert\le 1,\ l\le k\}$.
Then, we have $mP=\bigcup_{j=0}^{2q}Q_j^m$.
\end{mythm}
\begin{proof}
Let $r'\in Q_j^m$.
For $i>j$, we have $r_i=r_i'+r_{i-1}=r_i'+(r_{i-1}'+r_{i-2})=\cdots=r_i'+\cdots +r_{j+1}'+m$, and for $i< j$, we have $-r_i=r_{i+1}'-r_{i+1}=r_{i+1}'+(r_{i+2}'-r_{i+2})=\cdots=r_{i+1}'+\cdots+r_j'-m$.
Therefore, we have
\begin{align*}
&\bigcup_{j=0}^{2q}Q_j^m\\
&=\bigcup_{j=0}^{2q}\bigg\{r\in\mathbb{R}^{2q}\,\bigg\vert\,r_i=\sum_{l=j+1}^ir_{l}'+m\ (j<i\le 2q),\ r_i=-\sum_{l=i+1}^jr_{l}'+m\ (0\le i< j),\\
&\qquad\qquad 0\le r_i\le m\ (i=0,\ldots,j-1,j+1,\ldots 2q)\}\\
&=\bigcup_{j=0}^{2q}\bigg\{r'\in\mathbb{R}^{2q}\,\bigg\vert\,-m\le \sum_{l=j+1}^ir_{l}'\le 0\ (j<i\le 2q),\ 0\le \sum_{l=i+1}^jr_l'\le m\ (0\le i< j)\bigg\}\\
&=\bigg\{r\in\mathbb{R}^{2q}\,\bigg\vert\,\bigg\vert \sum_{i=l}^kr_i\bigg\vert\le 1,\ l\le k\bigg\},
\end{align*}
which completes the proof of the lemma.
\end{proof}

\begin{mythm}[Proposition~\ref{prop:pd_poly_sep}]
The kernel $k_n^{\opn{poly},q}$ is a positive definite kernel.
\end{mythm}
\begin{proof}
%\color{red}
By the definition of $S_n$, we have $S_n(A^*)(z)=\overline{S_n(A)(z)}$.
Thus, we have
\begin{align*}
k_n^{\opn{poly},q}(x,y)&=S_n\bigg(\sum_{i=1}^d\alpha_i(R_n(x_i)^*)^qR_n(y_{i})^q\bigg)\\
&=S_n\bigg(\sum_{i=1}^d\alpha_i(R_n(y_{i})^*)^qR_n(x_i)^q\bigg)^*
=k_n^{\opn{poly},q}(y,x)^*
\end{align*}
for $x,y\in\alg^d$.
%\color{black}
In addition, let $x_1,\ldots,x_N\in\alg^d$ and $d_1,\ldots,d_N\in\alg$.
Then, we have
\begin{align*}
&\bigg(\sum_{j,l=1}^Nd_j^*k_n^{\opn{poly},q}(x_j,x_l)d_l\bigg)(z)\\
&=\sum_{j,l=1}^N\overline{d_j(z)}\sum_{j',l'=1}^n\sum_{i=1}^d\alpha_i\sum_{m=1}^n\big((R_n(x_{i,j})^*)^q\big)_{j',m}\big(R_n(x_{i,l})^q\big)_{m,l'}\mr{e}^{\mr{i}(j'-l')z}d_l(z)\\
&=\sum_{i=1}^d\alpha_i\sum_{m=1}^n\bigg\vert \sum_{j=1}^N\sum_{j'=1}^nd_j(z)\big((R_n(x_{i,j}))^q\big)_{m,j'}\mr{e}^{-\mr{i}j'z}\bigg\vert^2\ge 0
\end{align*}
for $z\in\mathbb{T}$.
%In addition, let $\tilde{x}_{n,j}=S_n(R_n(x_j))$. Then, we have
%\begin{align*}
%&\bigg(\sum_{j,l=1}^Nd_j^*k_n^{\opn{sep},q}(x_j,x_l)d_l\bigg)(z)\\
%&=\sum_{j,l=1}^N\overline{d_j(z)}\tilde{k}(\tilde{x}_{n,j},\tilde{x}_{n,l})d_l(z)\sum_{j',l'=1}^n
%\sum_{m_1,\ldots,m_{2q-1}=1}^n\\%\!\!\!\!\!
%&\big((R_n(a)^*)\big)_{j',m_1}\big((R_n(a)^*)\big)_{m_1,m_2}\cdots\big((R_n(a)^*)\big)_{m_{q-1},m_q}\big(R_n(a)\big)_{m_{q},m_{q+1}}\cdots\big(R_n(a)\big)_{m_{2q-1},l'}\mr{e}^{\mr{i}(j'-l')z}\\
%&=\sum_{j,l=1}^N\overline{d_j(z)}\tilde{k}(\tilde{x}_{n,j},\tilde{x}_{n,l})d_l(z)\\
%&\qquad\sum_{m_q=1}^n\bigg\vert \sum_{j'=1}^n\sum_{m_1,\ldots,m_{q-1}=1}^n\big(R_n(a)\big)_{m_q,m_1}\cdots\big(R_n(a)\big)_{m_{q-2},m_{q-1}}\big(R_n(a)\big)_{m_{q-1},j'}\mr{e}^{-\mr{i}j'z}\bigg\vert^2\\
%&\ge 0.
%\end{align*}
Thus, $k_n^{\opn{poly},q}$ is positive definite.
\end{proof}

\begin{mythm}[Proposition~\ref{prop:pd_prod}]
Let $\beta_n\ge-\min_{z\in\red{\mathbb{T}^{2q}}}F_n^{2q,P}(z)$.
Then, $\hat{k}_n^{\opn{prod},q}$ defined below is a positive definite kernel.
\begin{align*}
\hat{k}_n^{\opn{prod},q}(x,y)=k_n^{\opn{prod},q}(x,y)+\beta_n\int_{\mathbb{T}^{2q}}\prod_{j=1}^q\overline{\tilde{k}_{1,j}(x(t_j),y(t_j))}\tilde{k}_{2,j}(x(t_{q+j}),y(t_{q+j}))\mr{d}t.
\end{align*}
\end{mythm}
\begin{proof}
%\color{red}
Since $F_n^{2q,P}$ is real-valued, we have
\begin{align*}
\hat{k}_n^{\opn{prod},q}(x,y)(z)
&=\int_{\mathbb{T}^{2q}}\prod_{j=1}^q\overline{\tilde{k}_{1,j}(x(t_j),y(t_j))}\tilde{k}_{2,j}(x(t_{q+j}),y(t_{q+j}))F_n^{2q,P}(z\mathbf{1}-t)\mr{d}t\\
&\qquad+\beta_n\int_{\mathbb{T}^{2q}}\prod_{j=1}^q\overline{\tilde{k}_{1,j}(x(t_j),y(t_j))}\tilde{k}_{2,j}(x(t_{q+j}),y(t_{q+j}))\\
&=\int_{\mathbb{T}^{2q}}\prod_{j=1}^q\tilde{k}_{1,j}(y(t_j),x(t_j))\overline{\tilde{k}_{2,j}(y(t_{q+j}),x(t_{q+j}))}F_n^{2q,P}(z\mathbf{1}-t)\mr{d}t\\
&\qquad+\beta_n\int_{\mathbb{T}^{2q}}\prod_{j=1}^q\tilde{k}_{1,j}(y(t_j),x(t_j))\overline{\tilde{k}_{2,j}(x(t_{q+j}),y(t_{q+j}))}\\
&=\overline{\hat{k}_n^{\opn{prod},q}(y,x)(z)}
\end{align*}
for $x,y\in\alg^d$.
%\color{black}
In addition, let $x_1,\ldots,x_N\in\alg^d$ and $d_1,\ldots,d_N\in\alg$.
Then, we have
\begin{align*}
&\bigg(\sum_{j,l=1}^Nd_j^*\hat{k}_n^{\opn{prod},q}(x_j,x_l)d_l\bigg)(z)\\
&=\sum_{m,l=1}^N\int_{\mathbb{T}^{2q}}\overline{d_m(z)}\bigg(\prod_{j=1}^q\overline{\tilde{k}_{1,j}(x_m(t_j),x_l(t_j))}\tilde{k}_{2,j}(x_m(t_{q+j}),x_l(t_{q+j}))\bigg)d_l(z)F_n^{2q,P}(z\mathbf{1}-t)\mr{d}t\\
&\qquad+\beta_n\sum_{m,l=1}^N\int_{\mathbb{T}^{2q}}\overline{d_m(z)}\bigg(\prod_{j=1}^q\overline{\tilde{k}_{1,j}(x_m(t_j),x_l(t_j))}\tilde{k}_{2,j}(x_m(t_{q+j}),x_l(t_{q+j}))\bigg)d_l(z)\mr{d}t\\
&=\sum_{m,l=1}^N\int_{\mathbb{T}^{2q}}\overline{d_m(z)}\bigg(\prod_{j=1}^q\overline{\tilde{k}_{1,j}(x_m(t_j),x_l(t_j))}\tilde{k}_{2,j}(x_m(t_{q+j}),x_l(t_{q+j}))\bigg)d_l(z)(F_n^{2q,P}(z\mathbf{1}-t)+\beta_n)\mr{d}t\\
&\ge 0
\end{align*}
for $z\in\mathbb{T}$.
The last inequality is derived since the map $([x_1,\ldots,x_{2q}],[y_1,\ldots,y_{2q}])\mapsto \prod_{j=1}^q\overline{\tilde{k}_{1,j}(x_{2j-1},y_{2j-1})}\tilde{k}_{2,j}(x_{2j},y_{2j})$ is positive definite and $F_n^{2q,P}(z\mathbf{1}-t)+\beta_n\ge 0$.
\end{proof}

\begin{mythm}[Lemma~\ref{lem:fejer_bound}]
We have $\vert F_n^{q,P}(z)\vert\le n^q$.    
\end{mythm}
\begin{proof}
The number of terms in $F_n^{q,P}$ is equal to the number of terms in $\sum_{m,l=0}^{n-1}(T_1\cdots T_{2q})_{m,l}$, where $T_1,\ldots,T_{2q}$ are Toeplitz matrices.
Thus, we have
\begin{align*}
\vert F_n^{q,P}(z)\vert
\le \frac1n n^{q+1}
= n^q.
\end{align*}
\end{proof}

Let $\Omega$ be a probability space with a probability measure $\mu$.
Let $X_1,\ldots,X_N$ and $Y_1,\ldots,Y_N$ be samples from a distributions of $\alg_0^d$-valued random variable $X$ and $\alg_1$-valued random variable $Y$ on $\Omega$, respectively (i.e., for $z\in\mathbb{T}$, $X_i(z)$ is a sample from the distribution of $X(z)$).
Here, $\alg_0$ and $\alg_1$ are subsets of $\alg$.

\begin{mythm}[Theorem~\ref{thm:generalization}]
Assume $k_n^{\opn{poly},q}$ and $\hat{k}_n^{\opn{prod},q}$ are real-valued.
Let $D(k_n^{\opn{poly},q},x)=\sum_{j=1}^d\alpha_j\Vert R_n(x_{j})\Vert^{2q}_{\opn{op}}$ and $D(\hat{k}_n^{\opn{prod},q},x)=\prod_{j=1}^q(\Vert R_n(\tilde{k}_{1,j}(x,x))\Vert_{\opn{op}}\Vert R_n(\tilde{k}_{2,j}(x,x))\Vert_{\opn{op}})+\beta_n C$, 
%and $D(k_n^{\opn{sep},q},x)=\tilde{k}(x,x)\prod_{j=1}^q\Vert R_n(a_j)\Vert_{\opn{op}}^2$ for $x\in\alg_0^d$, 
where $\Vert\cdot\Vert_{\opn{op}}$ is the operator norm and $C=\prod_{j=1}^q\int_{\mathbb{T}}\tilde{k}_{1,j}(x(t),x(t))\mr{d}t\int_{\mathbb{T}}\tilde{k}_{2,j}(x(t),x(t))\mr{d}t$.
\red{Assume $\Vert {k}_n^{\opn{poly},q}(x,x)\Vert_{\alg} < \infty$ for any $x\in\mcl{A}_0^d$ and $\Vert \hat{k}_n^{\opn{prod},q}(x,x)\Vert_{\alg}\le \infty$.}
For $k_n=k_n^{\opn{poly},q},\hat{k}_n^{\opn{prod},q}$, \red{for any $z\in\mathbb{T}$,} and
for any $\delta\in (0,1)$, with probability at least $1-\delta$, \red{simultaneously for all $f\in\mathcal{F}$,} we have 
\begin{align}
\red{\mr{E}[g(f(X),Y)](z)
\le_{\alg} \frac1N\sum_{i=1}^Ng(f(X_i),Y_i)(z)
+2L\frac{B}{N}\bigg(\sum_{i=1}^ND(k_n,X_i)\bigg)^{1/2}
+3\sqrt{\frac{\log 1/\delta}{N}}}.%\label{eq:generalization_bound}
\end{align} 
\end{mythm}

To show Theorem~\ref{thm:generalization}, we use the $\alg$-valued Rademacher complexity defined by~\citet{hashimoto23_aistats}.
%We need the following notion.
Let $\sigma_1,\ldots,\sigma_N$
%\todo{in which space are sigma} 
be i.i.d. Rademacher variables, {which take their values on $\{-1,1\}$}.
For a real-valued function class $\mcl{F}$ and $\mathbf{x}=[x_1,\ldots,x_N]\in(\alg^d)^N$, the empirical Rademacher complexity $\hat{R}_N(\mathbf{x},\mcl{F})$ is defined by
%\begin{align*}
$\hat{R}_N(\mathbf{x},\mcl{F})=\mr{E}[\sup_{f\in \mcl{F}}\sum_{i=1}^Nf(x_i)\sigma_i]/N$.
Here, $\mr{E}$ is the integration on $\Omega$ with respect to $\mu$.
Similar to the case of the standard Rademacher complexity, an $\alg$-valued version of the Rademacher complexity is defined as
$\hat{R}_{\alg,N}(\mathbf{x},\mcl{F})=\mr{E}[\sup^{\alg}_{f\in \mcl{F}}\sum_{i=1}^N\vert f(x_i)^*\sigma_i\vert_{\alg}]/N$ for an $\alg$-valued function class $\mcl{F}$.
Here, $\vert a\vert_{\alg}=(a^*a)^{1/2}$ for $a\in\alg$, $\sup^{\alg}$ is the supremum in the sense of the order in $\alg$ (see Definition~\ref{def:sup}), and the integral $\mr{E}$ means the Bochner integral in this case.

Indeed, in our case, the $\alg$-valued Rademacher complexity is represented by the standard Rademacher complexity, and the argument is reduced to evaluate the standard Rademacher complexity.
In the following, we denote by $\alg_+$ the subset of $\alg$ composed of all positive elements in $\alg$.

\begin{proposition}\label{prop:rademacher_equiv}
Let $\mcl{F}$ be an $\alg$-valued function class.
Assume for any $x\in\alg_0^d$ and $z\in\mathbb{T}$, $f(x)(z)\in\mathbb{R}$.
Then, we have $\hat{R}_N(\mathbf{x},\mcl{F}(z))=\hat{R}_{\alg,N}(\mathbf{x},\mcl{F})(z)$, where $\mcl{F}(z)=\{x\mapsto (f(x))(z)\,\mid\,f\in\mcl{F}\}$.
\end{proposition}
\begin{proof}
We first show $(\sup^{\alg}_{a\in\mcl{S}}a)(z)=\sup_{a\in\mcl{S}}a(z)$ for $\mcl{S}\subseteq \alg_+$ for any $z\in\mathbb{T}$.
Let $a\in\mcl{S}$.
Since $a(z)\le \sup_{a\in\mcl{S}}a(z)$ for any $z\in\mathbb{T}$, we have $a\le_{\alg} b$, where $b\in\alg$ is defined as $b(z)=\sup_{a\in\mcl{S}}a(z)$.
Thus, $b$ is an upper bound of $\mcl{S}$, and we have $\sup^{\alg}_{a\in\mcl{S}}a\le_{\alg} b$, which means $(\sup^{\alg}_{a\in\mcl{S}}a)(z)\le\sup_{a\in\mcl{S}}a(z)$ for any $z\in\mathbb{T}$.
Conversely, since $a\le_{\alg} \sup^{\alg}_{a\in\mcl{S}}a$ for $a\in\mcl{S}$, we have $a(z)\le (\sup^{\alg}_{a\in\mcl{S}}a)(z)$ for any $z\in\mathbb{T}$.
Thus, we have $\sup_{a\in\mcl{S}}a(z)\le (\sup^{\alg}_{a\in\mcl{S}}a)(z)$.
Therefore, we have
\begin{align*}
\hat{R}_{\alg,N}(\mathbf{x},\mcl{F})(z)
&=\frac1N\mr{E}\bigg[\sup_{f\in \mcl{F}}\!^{\alg}\bigg\vert\sum_{i=1}^Nf(x_i)^*\sigma_i\bigg\vert_{\alg}\bigg](z)
=\frac1N\mr{E}\bigg[\bigg(\sup_{f\in \mcl{F}}\!^{\alg}\bigg\vert\sum_{i=1}^Nf(x_i)^*\sigma_i\bigg\vert_{\alg}\bigg)(z)\bigg]\\
&=\frac1N\mr{E}\bigg[\sup_{f\in \mcl{F}(z)}\bigg\vert \sum_{i=1}^N\overline{f(x_i)}\sigma_i\bigg\vert\bigg]
=\frac1N\mr{E}\bigg[\sup_{f\in \mcl{F}(z)}\bigg\vert\sum_{i=1}^Nf(x_i)\sigma_i\bigg\vert\bigg]
=\hat{R}_N(\mathbf{x},\mcl{F}(z)),
\end{align*}
where the third equality is given by the identity $\vert a\vert_{\alg}(z)=\vert a(z)\vert$, and the forth equality is satisfied since $\sigma_i$ is real-valued.
\end{proof}

We obtain the following lemma directly by Lemma 4.2 in~\citet{mohri18}.
\begin{lemma}\label{lem:rademacher_transform}
Let $\mcl{F}$ be an $\alg$-valued function class.
Let $g:\mathbb{R}\times\mathbb{R}\to\mathbb{R}_{+}$ be a \red{bounded} error function. %$h_i(t):\mathbb{C}\to\mathbb{R}_+$ be defined as $h_i(t)(x)=g(x,y_i(t))$.
Assume there exists $L>0$ such that for $y\in\alg_1$ and $z\in\mathbb{T}$, $x\mapsto g(x,y(z))$ is $L$-Lipschitz continuous.
Assume also for any $x\in\alg_0^d$ and $z\in\mathbb{T}$, $f(x)(z)\in\mathbb{R}$.
Then, we have
\begin{align*}
\frac1N\mr{E}\bigg[\sup_{f\in \mcl{F}}\sum_{i=1}^Ng(f(x_i)(z),y_i(z))\sigma_i\bigg]
=\frac1N\mr{E}\bigg[\sup_{f\in \mcl{F}(z)}\sum_{i=1}^Ng(f(x_i),y_i(z))\sigma_i\bigg]\le L\hat{R}_N(\mathbf{x},\mcl{F}(z)).
\end{align*}
\end{lemma}

We apply the following lemma to obtain the inequality~\eqref{eq:generalization_bound}.
\begin{lemma}
Let $\mcl{F}$ be an $\alg$-valued function class.
Let $g$ be the same map defined in Lemma~\ref{lem:rademacher_transform}. 
For any $\delta\in (0,1)$, with probability at least $1-\delta$, \red{simultaneously for all $f\in\mathcal{F}$,} we have
\begin{align*}
\mr{E}[g(f(X),Y)](z)
\le \frac1N\sum_{i=1}^Ng(f(X_i),Y_i)(z)
+\frac2N\mr{E}_{\sigma_i}\bigg[\sup_{f\in \mcl{F}}\sum_{i=1}^Ng(f(X_i)(z),Y_i(z))\sigma_i\bigg]
+3\sqrt{\frac{\log 1/\delta}{N}}.
\end{align*}
Here, $\mr{E}_{\sigma_i}[\sup_{f\in \mcl{F}}\sum_{i=1}^Ng(f(X_i)(z),Y_i(z))\sigma_i]=\int_{\Omega}\sup_{f\in \mcl{F}}\sum_{i=1}^Ng(f(X_i)(z),Y_i(z))\sigma_i(\omega)\mr{d}\mu(\omega)$.
\end{lemma}
\begin{proof}
Note that we have $g(f(X_i)(z),Y_i(z))=g(f(X_i),Y_i)(z)$.
Therefore, with probability at least $1-\delta$, we have
\begin{align*}
&\mr{E}[g(f(X),Y)](z)=\mr{E}[g(f(X)(z),Y(z))]\\
&\quad\le \frac1N\sum_{i=1}^Ng(f(X_i)(z),Y_i(z))
+\frac2N\mr{E}_{\sigma_i}\bigg[\sup_{f\in \mcl{F}}\sum_{i=1}^Ng(f(X_i)(z),Y_i(z))\sigma_i\bigg]
+3\sqrt{\frac{\log 1/\delta}{N}}.
%&\quad =\frac1N\sum_{i=1}^Ng(f(X_i),Y_i)(z)
%+\frac2N\mr{E}_{\sigma_i}\bigg[\sup_{f\in \mcl{F}}\sum_{i=1}^Ng(f(X_i)(z),Y_i(z))\sigma_i\bigg]
%+3\sqrt{\frac{\log 1/\delta}{N}}.
\end{align*}
The inequality follows by Theorem 3.1 in~\cite{mohri18}.
\end{proof}

According to \citet{hashimoto23_aistats}, the $\alg$-valued Rademacher complexity for the RKHM associated with an $\alg$-valued kernel $k$ is upperbounded by $\Vert k(x_i,x_i)\Vert_{\alg}$ as follows, which completes the proof of Theorem~\ref{thm:generalization}.
\begin{lemma}
Let $B>0$ and $\mcl{F}=\{f\in\modu_k\,\mid\,\Vert f\Vert_k\le B\}$.
Then, we have 
\begin{align*}
\hat{R}_{\alg,N}(\mathbf{x},\mcl{F})(z)
\le\frac{B}{N}\bigg(\sum_{i=1}^N\Vert k(x_i,x_i)\Vert_{\alg}\bigg)^{1/2}.
\end{align*}
\end{lemma}

\begin{mythm}[Lemma~\ref{lem:factor}]
Assume $\beta_n\le \beta_{n+1}$ for $\hat{k}_n^{\opn{prod},q}$.
With the same assumptions as Theorem~\ref{thm:generalization}, for $k_n=k_n^{\opn{poly},q},\hat{k}_n^{\opn{prod},q}$ and for $x\in\alg_0^d$, we have $D(k_n,x)\le D(k_{n+1},x)$.
\end{mythm}
\begin{proof}
We evaluate $\Vert k_n(x,x)\Vert_{\alg}$.
For $k_n^{\opn{poly},q}$, we have
\begin{align*}
\Vert k_n^{\opn{poly},q}(x,x)\Vert
&=\bigg\Vert S_n\bigg(\sum_{i=1}^d\alpha_i(R_n(x_i)^*)^qR_n(x_i)^q\bigg)\bigg\Vert_{\alg}
\le \bigg\Vert \sum_{i=1}^d\alpha_i(R_n(x_i)^*)^qR_n(x_i)^q\bigg\Vert_{\opn{op}}\\
&\le \sum_{i=1}^d\alpha_i\Vert R_n(x_i)\Vert^{2q}_{\opn{op}}.
\end{align*}
The calculation for $\hat{k}_n^{\opn{prod},q}$ is similar.
The first inequality is derived by the following inequalities for $A\in\mathbb{C}^{n\times n}$:
\begin{align*}
\vert S_n(A)(z)\vert
&=\bigg\vert\frac1n\sum_{j,l=0}^{n-1}A_{j,l}\mr{e}^{\mr{i}(j-l)z}\bigg\vert\\
&=\bigg\vert \opn{tr}\bigg(\frac1n \mathbf{1}\opn{diag}(\mr{e}^0,\ldots,\mr{e}^{\mr{i}(n-1)})A\opn{diag}(\mr{e}^{0z},\ldots,\mr{e}^{\mr{i}(n-1)z})\bigg)\bigg\vert\\
&\le \frac1n\Vert \mathbf{1}\Vert_1\Vert \opn{diag}(\mr{e}^0,\ldots,\mr{e}^{\mr{i}(n-1)})A\opn{diag}(\mr{e}^{0z},\ldots,\mr{e}^{\mr{i}(n-1)z})\Vert_{\opn{op}}
\le \Vert A\Vert_{\opn{op}}
\end{align*}
for any $z\in\mathbb{T}$.
Here, $\opn{tr}(A)$ is the trace of a matrix $A$, $\opn{diag}(a_1,\ldots,a_n)$ for $a_1,\ldots,a_n\in\mathbb{C}$ is the diagonal matrix whose diagonal elements are $a_1,\ldots,a_n$, $\mathbf{1}\in\mathbb{C}^{n\times n}$ is the matrix whose elements are all $1$, and $\Vert A\Vert_1$ for $A\in\mathbb{C}^{n\times n}$ is the trace norm defined as $\Vert A\Vert_1=\opn{tr}((A^*A)^{1/2})$.
We used the inequality $\vert\opn{tr}(AB)\vert\le \Vert A\Vert_1\Vert B\Vert_{\opn{op}}$ for $A,B\in\mathbb{C}^{n\times n}$~\cite[Chapter IX, Section 2]{conway07}.

In addition, there exists $v\in\mathbb{C}^{n\times n}$ such that $\Vert R_n(x)\Vert_{\opn{op}}=\Vert R_n(x)v\Vert$ and $\Vert v\Vert=1$.
Thus, we have
\begin{align*}
\Vert R_n(x)\Vert_{\opn{op}}&
=\Vert R_n(x)v\Vert
=\Vert Q_n^*M_xQ_nv\Vert
=\bigg\Vert Q_n^*M_xQ_{n+1}\begin{bmatrix}
v\\
0
\end{bmatrix}\bigg\Vert
\le \bigg\Vert Q_{n+1}^*M_xQ_{n+1}\begin{bmatrix}
v\\
0
\end{bmatrix}\bigg\Vert\\
&\le \Vert Q_{n+1}^*M_xQ_{n+1}\Vert_{\opn{op}}
=\Vert R_{n+1}(x)\Vert_{\opn{op}},
\end{align*}
where $Q_n=[e_1,\ldots,e_n]$ and $e_j(z)=\mr{e}^{\mr{i}jz}$ is the Fourier function.
\end{proof}

%\color{blue}
\begin{mythm}[Theorem~\ref{thm:convergence_Td}]
For $x,y\in\alg^d$ and $z\in\mathbb{T}$, $k_n^{\opn{poly},q}(x,y)(z)\to k^{\opn{poly},q}(x,y)(z)$ and $k_n^{\opn{prod},q}(x,y)(z)\to k^{\opn{prod},q}(x,y)(z)$ as $n\to\infty$.
\end{mythm}
\begin{proof}
We show $k_n^{\opn{poly},q}(x,y)(z)\to k^{\opn{poly},q}(x,y)(z)$.
The proof for $k_n^{\opn{prod},q}(x,y)$ is similar.
For simplicity, we omit the subscript $i$ in Definition~\ref{def:kernels}.
We have
\begin{align*}
&k_n^{\opn{poly},q}(x,y)(z)
=S_n\bigg(\sum_{i=1}^d\alpha_i(R_n(x_i)^*)^qR_n(y_i)^q\bigg)(z)\\
&=\frac1{n}\sum_{j,l\in N}\sum_{r_1,\ldots,r_{2q-1}\in N}R_n(x)^*_{j,r_1}\cdots R_n(x)^*_{r_{q-1},r_q} R_n(y)_{r_q,r_{q+1}}\cdots R_n(y)_{r_{2q-1},l}\mr{e}^{\mr{i}(j-l)\cdot z}\\
&=\frac1{n}\sum_{j,l\in N}\sum_{r_1,\ldots,r_{2q-1}\in N}\int_{\mathbb{T}^d}\cdots\int_{\mathbb{T}^d}\overline{x(t_1)}\cdots \overline{x(t_q)}y(t_{q+1})\cdots y(t_{2q})\\
&\qquad\cdot \mr{e}^{\mr{i}(r_1-j)\cdot t_1}\cdots\mr{e}^{\mr{i}(r_{q}-r_{q-1})\cdot t_q}\mr{e}^{\mr{i}(r_{q+1}-r_{q})\cdot t_{q+1}}\cdots\mr{e}^{\mr{i}(l-r_{2q-1})\cdot t_{2q}}\mr{d}t_1\cdots\mr{d}t_{2q}\mr{e}^{\mr{i}(j-l)\cdot z}\\
&=\frac1{n}\sum_{j,l\in N}\sum_{r_1,\ldots,r_{2q-1}\in N}\int_{\mathbb{T}^d}\cdots\int_{\mathbb{T}^d}\overline{x(t_1)}\cdots \overline{x(t_q)}y(t_{q+1})\cdots y(t_{2q})\\
&\qquad\cdot \prod_{i=1}^d\mr{e}^{\mr{i}(r_{1,i}-j_i)(t_{1,i}-z_i)}\cdots\mr{e}^{\mr{i}(r_{q,i}-r_{q-1,i})(t_{q,i}-z_i)}\\
&\qquad\cdot \prod_{i=1}^d\mr{e}^{\mr{i}(r_{q+1,i}-r_{q,i})(t_{q+1,i}-z_i)}\cdots\mr{e}^{\mr{i}(l_i-r_{2q-1,i})(t_{2q,i}-z_i)}\mr{d}t_1\cdots\mr{d}t_{2q}\\
&=\frac1{n}\int_{\mathbb{T}^d}\cdots\int_{\mathbb{T}^d}\prod_{i=1}^d\sum_{j_i,l_i=0}^{n-1}\sum_{r_{1,i},\ldots,r_{2q-1,i}=0}^{n-1}\overline{x(t_1)}\cdots \overline{x(t_q)}y(t_{q+1})\cdots y(t_{2q})\\
&\qquad\cdot \mr{e}^{\mr{i}(r_{1,i}-j_i)(t_{1,i}-z_i)}\cdots\mr{e}^{\mr{i}(r_{q,i}-r_{q-1,i})(t_{q,i}-z_i)}\\
&\qquad\cdot \mr{e}^{\mr{i}(r_{q+1,i}-r_{q,i})(t_{q+1,i}-z_i)}\cdots\mr{e}^{\mr{i}(l_i-r_{2q-1,i})(t_{2q,i}-z_i)}\mr{d}t_1\cdots\mr{d}t_{2q}\\
&=\frac1{n}\int_{\mathbb{T}^d}\cdots\int_{\mathbb{T}^d}\prod_{i=1}^d\sum_{m=0}^{n-1}\bigg(\sum_{j_i,l_i,r_{1,i},\ldots,r_{2q-1,i}=0}^m-\sum_{j_i,l_i,r_{1,i},\ldots,r_{2q-1,i}=0}^{m-1}\bigg)\\
&\qquad\cdot \overline{x(t_1)}\cdots \overline{x(t_q)}y(t_{q+1})\cdots y(t_{2q})\mr{e}^{\mr{i}(r_{1,i}-j_i)(t_{1,i}-z_i)}\cdots\mr{e}^{\mr{i}(r_{q,i}-r_{q-1,i})(t_{q,i}-z_i)}\\
&\qquad\cdot \mr{e}^{\mr{i}(r_{q+1,i}-r_{q,i})(t_{q+1,i}-z_i)}\cdots\mr{e}^{\mr{i}(l_i-r_{2q-1,i})(t_{2q,i}-z_i)}\mr{d}t_1\cdots\mr{d}t_{2q}\\
&=\frac1{n}\int_{\mathbb{T}^d}\cdots\int_{\mathbb{T}^d}\prod_{i=1}^d\sum_{m=0}^{n-1}\sum_{\substack{j_i\lor l_i\lor r_{1,i}\lor\ldots\lor r_{2q-1,i}= m\\ 0\le j_i,l_i,r_{1,i},\ldots,r_{2q-1,i}\le m}}\\
&\qquad\cdot \overline{x(t_1)}\cdots \overline{x(t_q)}y(t_{q+1})\cdots y(t_{2q})\mr{e}^{\mr{i}(r_{1,i}-j_i)(t_{1,i}-z_i)}\cdots\mr{e}^{\mr{i}(r_{q,i}-r_{q-1,i})(t_{q,i}-z_i)}\\
&\qquad\cdot \mr{e}^{\mr{i}(r_{q+1,i}-r_{q,i})(t_{q+1,i}-z_i)}\cdots\mr{e}^{\mr{i}(l_i-r_{2q-1,i})(t_{2q,i}-z_i)}\mr{d}t_1\cdots\mr{d}t_{2q}\\
&=\frac1{n}\int_{\mathbb{T}^d}\cdots\int_{\mathbb{T}^d}\prod_{i=1}^d\sum_{j=0}^{n-1}\sum_{(r_{1,i},\ldots,r_{2q,i})\in jP\bigcap \mathbb{Z}^{2q}}\\
&\qquad\cdot \overline{x(t_1)}\cdots \overline{x(t_q)}y(t_{q+1})\cdots y(t_{2q})\mr{e}^{\mr{i}r_{1,i}(t_{1,i}-z_i)}\cdots\mr{e}^{\mr{i}r_{q,i}(t_{q,i}-z_i)}\\
&\qquad\cdot \mr{e}^{\mr{i}r_{q+1,i}(t_{q+1,i}-z_i)}\cdots\mr{e}^{\mr{i}r_{2q,i}(t_{2q,i}-z_i)}\mr{d}t_1\cdots\mr{d}t_{2q}\\
&=g\ast F_n^{2qd,P^d}(z),
\end{align*}
where the sum $\sum_{j,l,r_1,\ldots,r_{2q-1}=0}^{-1}$ is set as $0$, $g(t)=\overline{x(t_1)}\cdots \overline{x(t_q)}y(t_{q+1})\cdots y(t_{2q})$ for $t=[t_1,\ldots,t_{2q}]\in\mathbb{T}^{d\times 2q}$, and  $P=\{r=[r_1,\ldots,r_{2q}]\in\mathbb{R}^{2q}\,\mid\,\vert \sum_{i=l}^mr_i\vert\le 1,\ l\le m\}$.
In addition, we set $j=r_0$ and $j=r_q$.
%The second to the last equality is derived by Lemma~\ref{lem:index} below.
Since $P^d$ is a convex polyhedron, $g\ast F_n^{2qd,P^d}(z)\to g(z)$ as $n\to\infty$ by Lemma~\ref{lem:fejer_convergence_general}.
%For $k_n^{\opn{sep},q}$, we additionally use the fact $\lim_{n\to\infty}\tilde{k}(S_n(R_n(x)),S_n(R_n(y)))=\tilde{k}(\lim_{n\to\infty}S_n(R_n(x)),\lim_{n\to\infty}S_n(R_n(y)))$, which follows from the continuity of $\tilde{k}$, and the fact $S_n(R_n(x))\to x$ uniformly as $n\to\infty$ if $x$ and $y$ are differentiable~\cite[Lemma 10]{suijlekom21}.
\end{proof}
%\color{black}

\section{Relationship between the parameter $q$ and the interactions}\label{ap:q_interactions}
Theorem~\ref{thm:convergence} implies that the interactions along $z\in\mathbb{T}$ in the kernels become small as $n$ grows.
On the other hand, we also have the parameter $q$ that describes how many Toeplitz matrices created by $R_n$ are involved in the kernel.
Regarding the relationship between $q$ and the interactions, we have the following remark.
\begin{remark}\label{rmk:q_convergence}
We expect that the convergence in Theorem~\ref{thm:convergence} becomes slower as $q$ becomes larger, and the interaction along $z\in\mathbb{T}$ becomes larger.
This is based on the observation that the value related to the Fej\'{e}r kernel on $\mathbb{T}^q$ is described by the corresponding value related to the Fej\'{e}r kernel on $\mathbb{T}^{q-1}$, especially in the case of $q=2$.
\end{remark}

To understand the relationship between $q$ and the interactions, we observe the convergence shown in Lemma~\ref{lem:fejer_convergence_general} from the perspective of $q$.
We have the following lemma, which is also mentioned by~\citet{brandolini97}.
\begin{lemma}\label{lem:poisson}
Let $H_n^{q,P}(t)=n^q\int_0^1\int_{\mathbb{R}^q}\chi_{jP}(r)\mr{e}^{\mr{i}nr\cdot t}\mr{d}r\mr{d}j$ for $t\in\mathbb{T}^q$, where $\chi_S$ is the characteristic function with respect to a set $S$.
Then, we have $F_n^{q,P}(t)= \sum_{m\in\mathbb{Z}^q}H_n^{q,P}(t+2\pi m)$.
\end{lemma}
\begin{proof}
Let $u_{j,t}(r)=\chi_{jP}(r)\mr{e}^{\mr{i}r\cdot t}$.
Then, we have
\begin{align*}
F_n^{q,P}(t)&=\frac1n\sum_{j=1}^{n-1}\sum_{r\in jP\bigcap \mathbb{Z}^q}\mr{e}^{\mr{i}r\cdot t}
=\frac1n\int_0^n\sum_{r\in\mathbb{Z}^q}\chi_{jP}(r)\mr{e}^{\mr{i}r\cdot t}\mr{d}j
=\frac1n\int_0^n\sum_{r\in\mathbb{Z}^q}u_{j,t}(r)\mr{d}j\\
&=\frac1n\int_0^n\sum_{m\in\mathbb{Z}^q}\widehat{u_{j,t}}(m)\mr{d}j
=\sum_{m\in\mathbb{Z}^q}\frac1n\int_0^n\int_{\mathbb{R}^q}\chi_{jP}(r)\mr{e}^{\mr{i}r\cdot t}\mr{e}^{2\pi\mr{i}r\cdot m}\mr{d}r\mr{d}j\\
&=\sum_{m\in\mathbb{Z}^q}\int_0^1\int_{\mathbb{R}^q}\chi_{jnP}(r)\mr{e}^{\mr{i}r\cdot t}\mr{e}^{2\pi\mr{i}r\cdot m}\mr{d}r\mr{d}j
=\sum_{m\in\mathbb{Z}^q}\int_0^1\int_{\mathbb{R}^q}\chi_{jP}\bigg(\frac{r}{n}\bigg)\mr{e}^{\mr{i}r\cdot t}\mr{e}^{2\pi\mr{i}r\cdot m}\mr{d}r\mr{d}j\\
&=n^q\sum_{m\in\mathbb{Z}^q}\int_0^1\int_{\mathbb{R}^q}\chi_{jP}(r)\mr{e}^{\mr{i}nr\cdot t}\mr{e}^{2\pi\mr{i}nr\cdot m}\mr{d}r\mr{d}j,
\end{align*}
where $\hat{\cdot}$ is the Fourier transform, and the forth equality is by the Poisson summation formula.
\end{proof}

To understand the convergence, we split $F_n^{q,P}(t)=\sum_{m\in\mathbb{Z}^q}H_n^{q,P}(t+2\pi m)$ into two parts: $H_n^{q,P}(t)$ and $\tilde{H}_{n}^{q,P}(t):=\sum_{m\neq 0}H_n^{q,P}(t+2\pi m)$.
In the convergence discussed in Lemma~\ref{lem:fejer_convergence_general}, the convolution $g\ast {H}_{n}^{q,P}(z)$ goes to $g(z)$.
On the other hand, $g\ast \tilde{H}_{n}^{q,P}(z)$ goes to $0$~\cite[Theorem 1]{brandolini97}.
Regarding the convergence of $g\ast {H}_{n}^{q,P}(z)$ we can evaluate it as follows, especially for $q=2$.
\begin{lemma}\label{lem:speed}
For $z\in\mathbb{T}$, we have
\begin{align*}
\lim_{n\to\infty}g\ast H_n^{q,P}(z)-g(z)=\int_0^1\lim_{n\to\infty}\bigg(\int_{t\in\mathbb{T}^q}\int_{jnP}\mr{e}^{\mr{i}r\cdot (z-t)}g(t)\mr{d}r\mr{d}t-g(z)\bigg)\mr{d}j.
\end{align*}
Moreover, let $h^{P}_{s}(z)=\int_{t\in\mathbb{T}^q}\int_{sP}\mr{e}^{\mr{i}r\cdot (z-t)}g(t)\mr{d}r\mr{d}t-g(z)$.
Let $q=2$, $Q=\{(r_1,-r_2)\,\mid\,(r_1,r_2)\in P\}$, $P_1'=[-1,1]$, and $P_2=\{(r_1,r_2)\,\mid\,\vert r_1\vert\le 1,\ \vert r_2\vert\le 1,\ \vert r_1+r_2\vert\le 1,\ \vert r_1-r_2\vert\le 1\}$.
Let $g(z)=\overline{x_1(z_1)}\overline{x_2(z_2)}y_1(z_{3})y_4(z_{4})$ for $z=[z_1,z_2,z_3,z_4]$.
Then, for $z\in\mathbb{T}^2$, there exists $C(z)\ge 0$ such that we have
\begin{align}
\frac{1}{2}\vert h^{P}_{jn}(z)+h^{Q}_{jn}(z)\vert \le \frac{1}{2}C(z)\sum_{j=1}^2&\bigg\vert\int_{t\in\mathbb{T}}\int_{jnP_1'}\mr{e}^{\mr{i}r(z_1-t)}x_j(t_j)\mr{d}r\mr{d}t-x_j(z_j)\bigg\vert\nn\\
&+\frac{1}{2}\bigg\vert\int_{\mathbb{T}^2}\int_{jnP_2}\mr{e}^{\mr{i}r\cdot (z-t)}g(t)\mr{d}r\mr{d}t_1-g(z)\bigg\vert.\label{eq:speed}
\end{align}
\end{lemma}
\begin{proof}
We have
\begin{align*}
& n^q\int_{t\in\mathbb{T}^q}\int_0^1\int_{jP}\mr{e}^{\mr{i}nr\cdot (z-t)}g(t)\mr{d}r\mr{d}j\mr{d}t-g(z)
= \int_0^1\int_{t\in\mathbb{T}^q}\int_{jnP}\mr{e}^{\mr{i}r\cdot (z-t)}g(t)\mr{d}r\mr{d}t-g(z)\mr{d}j.
\end{align*}
%Let ${h}^P_{s}(z)=\int_{t\in\mathbb{T}}\int_{sP}\mr{e}^{\mr{i}r\cdot (z-t)}g(t)\mr{d}r\mr{d}t-g(z)$.
Then, $h^P_{0}(z)=0$, and since $\lim_{s\to\infty}h^P_{s}(z)=0$, there exists $D(z)>0$ such that $\vert h^P_s(z)\vert \le D(z)$ for any $s>0$.
Thus, $\vert h^P_{jn}(z)\vert \le D(z)$ for any $n\in\mathbb{N}$ and $0< j\le 1$.
Thus, by the bounded convergence theorem,
we have
\begin{align*}
\lim_{n\to\infty}\int_0^1\int_{t\in\mathbb{T}^q}\int_{jnP}\mr{e}^{\mr{i}r\cdot (z-t)}g(t)\mr{d}r\mr{d}t-g(z)\mr{d}j
=\int_0^1\lim_{n\to\infty}\int_{t\in\mathbb{T}^q}\int_{jnP}\mr{e}^{\mr{i}r\cdot (z-t)}g(t)\mr{d}r\mr{d}t-g(z)\mr{d}j.
\end{align*}
In addition, let $P_1=[-1,1]^2$ and $P_2=\{(r_1,r_2)\,\mid\,\vert r_1\vert\le 1,\ \vert r_2\vert\le 1,\ \vert r_1+r_2\vert\le 1,\ \vert r_1-r_2\vert\le 1\}$.
Then, we have
\begin{align*}
h^{P}_{jn}(z)+h^{Q}_{jn}(z)
=h^{P_1}_{jn}(z)+h^{P_2}_{jn}(z).
\end{align*}
For $h^{P_1}_{jn}(z)$, we have
\begin{align*}
\vert h^{P_1}_{jn}(z)\vert
&=\bigg\vert\int_{\mathbb{T}}\int_{-jn}^{jn}\mr{e}^{\mr{i}r(z_1-t)}x_1(t)\mr{d}r\mr{d}t\int_{\mathbb{T}}\int_{-jn}^{jn}\mr{e}^{\mr{i}r(z_2-t)}x_2(t)\mr{d}r\mr{d}t-g(z)\bigg\vert\\
&\le \bigg\vert\int_{\mathbb{T}}\int_{-jn}^{jn}\mr{e}^{\mr{i}r(z_1-t)}x_1(t)\mr{d}r\mr{d}t\bigg(\int_{\mathbb{T}}\int_{-jn}^{jn}\mr{e}^{\mr{i}r(z_2-t)}x_2(t)\mr{d}r\mr{d}t-x_2(z_2)\bigg)\bigg\vert\\
&\quad+\bigg\vert x_2(z_2)\bigg(\int_{\mathbb{T}}\int_{-jn}^{jn}\mr{e}^{\mr{i}r(z_1-t)}x_1(t)\mr{d}r\mr{d}t-x_1(z_1)\bigg)\bigg\vert.
\end{align*}
In the same manner as $h_s^P(z)$, there exists $\tilde{C}(z)>0$ such that $\vert \int_{\mathbb{T}}\int_{-jn}^{jn}\mr{e}^{\mr{i}r(z_1-t)}x_1(t)\mr{d}r\mr{d}t\vert\le \tilde{C}(z)$.
By setting $C(z)=\max\{\tilde{C}(z),\vert x_2(z_2)\vert\}$, we obtain the result.
\end{proof}

The left-hand side of Eq.~\eqref{eq:speed} is the average of $h_{jn}^{P}(z)$ and $h_{jn}^{Q}(z)$, where $P$ and $Q$ are symmetric with respect to the second coordinate $r_2$.
The first term in the right hand side of Eq.~\eqref{eq:speed} is described by the convergence of $h_{jn}^{[-1,1]}(z)$, which corresponds to the one-dimensional ($(q-1)$-dimensional) Fej\'{e}r kernel.
Since we have an additional term $(1/2)\vert\int_{\mathbb{T}^2}\int_{jnP_2}\mr{e}^{\mr{i}r\cdot (z-t)}g(t)\mr{d}r\mr{d}t-g(z)\vert$, if $C(z)>1$, the convergence of $h^{P}_{jn}(z)$ is expected to be slower than its one-dimensional counterpart $h^{[-1,1]}_{jn}(z)$.
Based on this observation, we expect that the convergence in Theorem~\ref{thm:convergence} becomes slower as $q$ becomes larger, and the interaction along $z\in\mathbb{T}$ becomes larger.

\section{Representer theorem for RKHMs over general $C^*$-algebras}\label{ap:representer}
Hashimoto et al.~\cite[Proposition 4.5]{hashimoto23_moea} showed an approximate representer theorem for RKHMs over von Neumann-algebras.
Here, von Neumann-algebras are $C^*$-algebras that is closed with respect to the strong operator topology.
Since $C(\mathbb{T})$ is not a von Neumann-algebra, for the arguments in Section~\ref{sec:application}, we generalize the theorem to that for general $C^*$-algebras as follows.
\begin{proposition}
Let $x_1,\ldots,x_N\in\alg^d$.
Let $g:\alg^N\to \alg_+$ be a continuous map, let $h:\alg_+\to\alg_+$ satisfy $h(c)\le_{\alg}h(d)$ for $c,d\in\alg_+$ with $c\le_{\alg} d$.
Let $L(f)=g(f(x_1),\ldots,f(x_N))+h(\vert f\vert_k)$ for $f\in\modu_k$.
If there exists a solution $f_0$ of the minimization problem $\min_{f\in\modu_k}L(f)$, then for any $\epsilon>0$, there exists $\tilde{f}$ that admits the representation of $\tilde{f}=\sum_{i=1}^N\phi(x_i)c_i$ and $\Vert L(\tilde{f})-L(f_0)\Vert\le \epsilon$.
\end{proposition}
\begin{proof}
For $f_1,f_2\in\modu_k$, let $\theta_{f_1,f_2}:\modu_k\to\modu_k$ defined as $\theta_{f_1,f_2}(v)=f_1\bracket{f_2,v}_k$ for $v\in\modu_k$.
Let $\mcl{N}_k$ be the submodule generated algebraically by $\{\phi(x_1),\ldots,\phi(x_N)\}$, let $\mcl{B}_0={\opn{Span}\{\theta_{f_1,f_2}\,\mid\,f_1,f_2\in\mcl{N}_k\}}$, and let $\mcl{B}=\overline{\mcl{B}_0}$.
In addition, let $\mcl{L}(\modu_k)$ be the (unital) $C^*$-algebra of adjointable $\alg$-linear operators on $\modu_k$ (see \cite[p8]{lance95}).
Since $\mcl{B}$ is a $C^*$-subalgebra of $\mcl{L}(\modu_k)$ (Note $\theta_{f_1,f_2}\theta_{f_3,f_4}=\theta_{f_1,f_4\bracket{f_3,f_2}}$ and $\theta_{f_1,f_2}^*=\theta_{f_2,f_1}$), $\mcl{B}_0$ has a net $\{b_i\}_i$ such that $0\le_{\mcl{L}(\modu_k)}b_i\le_{\mcl{L}(\modu_k)} 1$ and $\lim_{i}\theta_{f_1,f_2}b_i=\lim_{i}b_i\theta_{f_1,f_2}=\theta_{f_1,f_2}$ for any $f_1,f_2\in\mcl{N}_k$~\cite[Theorem I.4.8]{davidson96}.
For each $v\in\mcl{N}_k$, we have
\begin{align*}
\Vert v-b_iv\Vert_{k}^2
&=\Vert \bracket{(1-b_i)v,(1-b_i)v}_k\Vert_{\alg}
=\Vert \theta_{(1-b_i)v,(1-b_i)v}\Vert_{\mcl{L}(\modu_k)}\\
&=\Vert (1-b_i)\theta_{v,v}(1-b_i)^*\Vert_{\mcl{L}(\modu_k)}
=\Vert (1-b_i)\theta_{v,v}(1-b_i)\Vert_{\mcl{L}(\modu_k)}.
\end{align*}
Thus, we have $\lim_ib_iv=v$.
Since $g$ is continuous, we have
\begin{align*}
&\lim_{i}g(\bracket{\phi_k(x_1),b_if_0}_k,\ldots,\bracket{\phi_k(x_N),b_if_0}_k)
=\lim_{i}g(\bracket{b_i^*\phi_k(x_1),f_0}_k,\ldots,\bracket{b_i^*\phi_k(x_N),f_0}_k)\\
&=\lim_{i}g(\bracket{b_i\phi_k(x_1),f_0}_k,\ldots,\bracket{b_i\phi_k(x_N),f_0}_k)
=g(\bracket{\phi_k(x_1),f_0}_k,\ldots,\bracket{\phi_k(x_N),f_0}_k).
\end{align*}
Therefore, for $\epsilon>0$, there exists $i$ such that 
\begin{align*}
\Vert g((b_if_0)(x_1),\ldots,(b_if_0)(x_N))-g(f_0(x_1),\ldots,f_0(x_N))\Vert_{\alg}\le \epsilon.
\end{align*}
Since $b_i\ge_{\mcl{L}(\modu_k)} 0$, there exists $c \ge_{\mcl{L}(\modu_k)} 0$ such that $b_i=c^2$.
Thus, $b_i-b_i^2=c^2-c^4=c(1-c^2)c=c^*(1-b_i)c\ge_{\mcl{L}(\modu_k)} 0$.
Thus, $0\le_{\mcl{L}(\modu_k)}b_i^2\le_{\mcl{L}(\modu_k)}b_i\le_{\mcl{L}(\modu_k)}1$, and we have $\vert b_if_0\vert_k^2=\bracket{f_0,b_i^2f_0}_k\le_{\alg}\bracket{f_0,f_0}_k=\vert f_0\vert_k^2$.
As a result, we obtain
\begin{align*}
0&\le_{\alg}L(b_if_0)-L(f_0)\\
&=
g((b_if_0)(x_1),\ldots,(b_if_0)(x_N))+h(\vert b_if_0\vert_k)
-g(f_0(x_1),\ldots,f_0(x_N))-h(\vert f_0\vert_k)\\
&\le_{\alg}g((b_if_0)(x_1),\ldots,(b_if_0)(x_N))-g(f_0(x_1),\ldots,f_0(x_N)).
\end{align*}
Since $b_if_0\in\mcl{N}_k$, setting $\tilde{f}=b_if_0$ completes the proof of the proposition.
\end{proof}

\begin{figure}[t]
    \centering
    \includegraphics[scale=0.3]{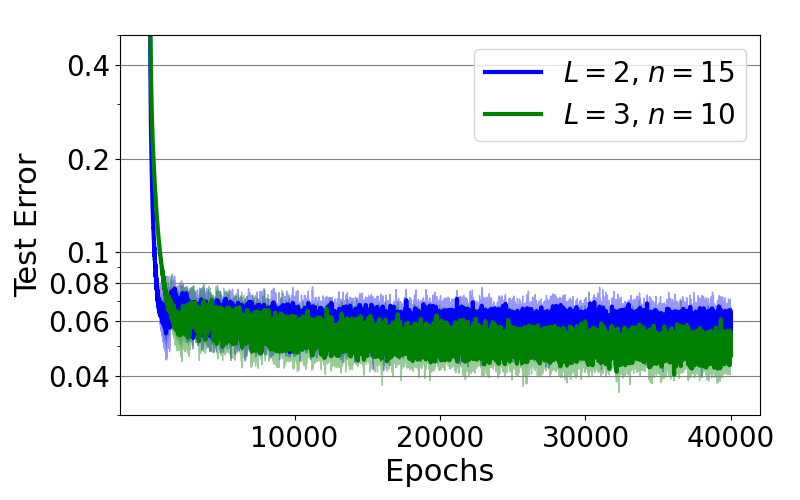}
    \includegraphics[scale=0.3]{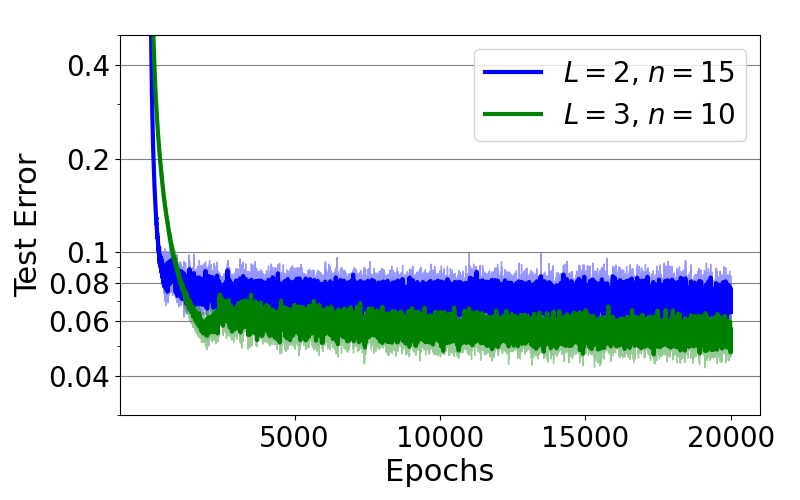}
    \caption{Test error of the regression task with the \blue{combined} approach with different $n$ and $L$. The parameters $n$ and $L$ are chosen so that the numbers of parameters are the same for all the cases. (Average value $\pm$ the standard deviation of three independent runs.) }
    \label{fig:deep}
\end{figure}

%\section{Experimental Details and Additional Results}\label{sec:exp_detail}
%We provide experimental details below.
%All the experiments in this paper were executed with Python 3.9 on an Intel(R) Core(TM) i9-10885H
%2.4GHz processor with the Windows 10 operating system.

%\color{teal}
\section{Additional numerical results}\label{ap:additional_numexp}
We observed the behavior of the \blue{combined model} to show its high representation power.
We generated data in the same manner as Subsection~\ref{subsec:exp_syn}, but we set $N=300$.
We used the kernel ${k}_n^{\opn{prod},q}$ and set $L=2,3$, $\tau_1=\sqrt{2}$, $\tau_2=\sqrt{3}$, and $\tau_3=\sqrt{5}/2$.
We set $n_j=n$ for $j=1,\ldots,L$ for simplicity and set $n$ so that the numbers of parameters are the same for all values of $L$.
We used the loss function $1/N\sum_{i=1}^N\Vert f(x_i)-y_i\Vert_{L^2(\mathbb{T})}$ for the \blue{combined model} $f$ defined in Eq.~\eqref{eq:deep}.
\blue{For the optimization, we used the Adam optimizer with learning rate $0.001$.}
Figure~\ref{fig:deep} shows the result.
We can see that as the number of layers increases, we can obtain a higher performance model even with the same number of parameters.

%\color{black}

\bigskip
\bibliography{main}

\end{document}